\documentclass[preprint,11pt,authoryear]{elsarticle}

\journal{European Journal of Operational Research}

\usepackage[letterpaper,margin=1in]{geometry}
\usepackage{setspace}
\usepackage{amsmath,amssymb,amsfonts,amsthm}
\usepackage{bm}
\usepackage{algorithm}
\usepackage{algpseudocode}
\usepackage{graphicx}
\graphicspath{{figs/}}
\usepackage{url}
\usepackage{tikz}
\usepackage{endnotes}
\usepackage[hidelinks]{hyperref}

\DeclareMathOperator{\diag}{diag}

\theoremstyle{plain}
\newtheorem{theorem}{Theorem}[section]
\newtheorem{lemma}[theorem]{Lemma}
\newtheorem{corollary}[theorem]{Corollary}
\newtheorem{proposition}[theorem]{Proposition}
\theoremstyle{definition}
\newtheorem{definition}[theorem]{Definition}
\newtheorem{assumption}[theorem]{Assumption}
\newtheorem{remark}[theorem]{Remark}

\biboptions{authoryear,round}
\hypersetup{
  pdftitle={An Accelerated Stochastic Variance-Reduced Algorithm for Entropic Wasserstein Barycenters},
  pdfauthor={Yiling Xie, Yiling Luo, Xiaoming Huo}
}

\begin{document}
\onehalfspacing

\begin{frontmatter}

\title{An Accelerated Stochastic Variance-Reduced Algorithm for Entropic Wasserstein Barycenters}

\author[cityu]{Yiling Xie\corref{cor1}}
\ead{yiling.xie@cityu.edu.hk}

\author[Google]{Yiling Luo}
\ead{luo\_mia@hotmail.com}

\author[gatech]{Xiaoming Huo}
\ead{huo@gatech.edu}

\cortext[cor1]{Corresponding author.}

\affiliation[cityu]{
  organization={Department of Decision Analytics and Operations, City University of Hong Kong},
  addressline={Tat Chee Avenue, Kowloon Tong, Kowloon},
  city={Hong Kong SAR},
  country={China}
}

\affiliation[Google]{
  organization={Google},
  city={Durham},
  state={NC},
  country={USA}
}

\affiliation[gatech]{
  organization={H. Milton Stewart School of Industrial and Systems Engineering, Georgia Institute of Technology},
  addressline={755 Ferst Drive NW},
  city={Atlanta},
  postcode={30332},
  state={GA},
  country={USA}
}

\begin{abstract}
Fixed-support Wasserstein barycenters average probability distributions while accounting for the
geometry of the support.
We study the entropically regularized Wasserstein barycenter problem with a fixed regularization parameter and
propose an accelerated stochastic variance-reduced primal-dual algorithm. The algorithm uses a
semi-dual finite-sum structure in which each stochastic gradient requires only one softmax over the
barycenter support. The resulting finite-sum components have dimension-free smoothness bounds, which
lead to a complexity result showing that the method improves the support-size dependence of
deterministic accelerated gradient by a square-root factor while preserving accelerated dependence on
the target accuracy. Experiments on synthetic data, DOTmark images, shape aggregation, and digit-averaging
instances are consistent with the theoretical dependence on support size and accuracy and show lower
arithmetic costs than the tested first-order baselines.
\end{abstract}


\end{frontmatter}

\section{Introduction}\label{intro}
Many problems in optimization and data analytics compare or average probability distributions.
Empirical laws define uncertainty sets in Wasserstein distributionally robust
optimization \citep{mohajerin2018data,blanchet2019quantifying}, while images, shapes, and point
clouds are often represented as histograms on a metric support in imaging and geometric data analysis
\citep{cuturi2014fast,solomon2015convolutional,peyre2019computational}. In these
settings, the geometry of the support is part of the averaging problem. Optimal transport encodes this
geometry by charging the movement of mass through a ground cost
\citep{monge1781memoire,kantorovich1942translocation,villani2009optimal,peyre2019computational}. The
corresponding Fr\'echet mean is the Wasserstein barycenter \citep{agueh2011barycenters}, a standard
tool for aggregating histograms and geometric data
\citep{rabin2011wasserstein,cuturi2014fast,solomon2015convolutional,peyre2019computational}.

The fixed-support Wasserstein barycenter is computationally demanding because the output distribution is unknown
and must be learned jointly with the transport plans from all input distributions. With inputs
$\mu_1,\ldots,\mu_m$ on a support of size $n$, the model contains $m$ transport plans, and these plans
are coupled by the requirement that their target marginals coincide. Entropic regularization turns the
discrete OT subproblems into smooth log-sum-exp objects and is widely used in large-scale OT computation
\citep{cuturi2013sinkhorn,feydy2019interpolating}. We study this fixed-support entropic Wasserstein barycenter
problem at a fixed regularization parameter $\eta$, treating the fixed-$\eta$ entropic objective as the
optimization target.

Known first-order rates for an $\epsilon$-accurate unregularized fixed-support Wasserstein barycenter include
iterative Bregman projections, or IBP,
$\widetilde{\mathcal O}\!\left(mn^2/\epsilon^2\right)$
\citep{benamou2015iterative,kroshnin2019complexity}, accelerated IBP, or FastIBP,
$\widetilde{\mathcal O}\!\left(mn^{7/3}/\epsilon^{4/3}\right)$
\citep{guminov2021combination}, and area-convexity and dual-extrapolation methods
$\widetilde{\mathcal O}\!\left(mn^2/\epsilon\right)$
\citep{dvinskikh2021improved,lin2020fixed}. For the fixed-$\eta$ entropic objective studied here, deterministic
accelerated gradient has rate
$\widetilde{\mathcal O}_\eta\!\left(mn^{2.5}/\sqrt{\epsilon}\right)$
\citep{kroshnin2019complexity}. Stochastic barycenter methods have also been developed, but the
available guarantees are stochastic-approximation guarantees without variance reduction
\citep{claici2018stochastic,li2020continuous}. These results leave open whether, in the
fixed-regularization regime, low-cost stochastic component gradients can be combined with acceleration
and variance reduction. Table~\ref{tab:bc-landscape} separates the unregularized and fixed-$\eta$
regimes; our rate comparison is the fixed-$\eta$ comparison with deterministic accelerated gradient.

We develop \textbf{PDASGD-BC}, an accelerated stochastic variance-reduced primal-dual method for the
fixed-support entropic Wasserstein barycenter. The method starts from a semi-dual reformulation in which the row
dual variables have been eliminated. The remaining dual variables enforce equality of the column
marginals of the recovered transport plans, which is the common barycenter marginal. For $m$ input
distributions with support size $n$, the semi-dual objective decomposes into $mn$ components, indexed
by pairs consisting of an input distribution and one of its support points. Evaluating one component
gradient requires only one softmax over the barycenter support, so an inner stochastic step costs
$\mathcal O(n)$, compared with the $\mathcal O(mn^2)$ cost of a full semi-dual gradient.

The analysis supplies the estimates needed to apply accelerated variance reduction to this
reformulation. We establish the barycenter analogue of the primal-gradient identity used for
two-marginal OT, verify component convexity and dimension-free smoothness bounds, and invoke the
general PDASGD result for linearly constrained convex optimization. Under the mass lower-bound
assumptions stated in the main barycenter theorem, PDASGD-BC returns a primal-dual output with
expected fixed-$\eta$ entropic barycenter objective residual at most $\epsilon$ in
$\widetilde{\mathcal O}_\eta\!\left(mn^2/\sqrt{\epsilon}\right)$ arithmetic operations. This
improves the deterministic accelerated-gradient rate
$\widetilde{\mathcal O}_\eta\!\left(mn^{2.5}/\sqrt{\epsilon}\right)$ by a factor $\sqrt n$ in
the same fixed-regularization regime.

The paper builds this result by moving from a general linearly constrained problem to two-marginal
entropic OT and then to barycenters. The two-marginal specialization also gives a reference case. For
fixed $\eta$, PDASGD attains
$\widetilde{\mathcal O}_\eta\!\left(n^2/\sqrt{\epsilon}\right)$ arithmetic complexity for the
entropic OT objective. The numerical experiments are consistent with the theory. In the two-marginal
tests, PDASGD is faster than APDAGD and AAM on the tested instances and is reported alongside Sinkhorn
as a specialized computational reference. On barycenter instances, PDASGD-BC recovers
closed-form Gaussian barycenters, forms shape and digit barycenters, matches IBP references on
image instances, and shows support-size behavior consistent with Theorem~\ref{bc:main} up to
support size $n=960$; see
Section~\ref{num}.

\subsection{Related Work}\label{relatedwork}
The relevant literature includes work on Wasserstein barycenters, computational OT solvers, entropic
OT and Sinkhorn variants, and first-order primal-dual algorithms for linearly constrained convex
optimization.

The OT problem has several variants, including unbalanced OT \citep{nguyen2023unbalanced},
semi-discrete OT \citep{tacskesen2023semi}, and multimarginal OT \citep{lin2022complexity}. We focus on
discrete OT, its entropy-regularized version, and the fixed-support entropic Wasserstein barycenter. Classical
solvers for discrete OT include network simplex, auction, and other combinatorial algorithms
\citep{peyre2019computational}; see \cite{xie2022solving} for a detailed discussion. Interior-point
methods provide another route and can use sparsity in the OT solution
\citep{zanetti2023interior,cipolla2024regularized}. Our approach targets the smooth entropically
regularized formulation and uses finite-sum structure in the dual.

Wasserstein barycenters were introduced as Fr\'echet means in Wasserstein space by
\citet{agueh2011barycenters}. Early computational and application-driven work used barycenters for
texture mixing, image and geometric-data averaging, and distribution clustering
\citep{rabin2011wasserstein,cuturi2014fast,solomon2015convolutional,ye2017fast}. Entropic and
Bregman-projection formulations made discrete barycenters computationally practical
\citep{cuturi2014fast,benamou2015iterative,cuturi2016smoothed}. A separate line studies complexity
for fixed-support Wasserstein barycenters. \citet{kroshnin2019complexity} analyze the cost of approximating
Wasserstein barycenters, \citet{lin2020fixed} establish hardness and fast algorithms for the
fixed-support case, and later work improves first-order bounds or studies distributed computation
\citep{guminov2021combination,dvinskikh2021improved,uribe2018distributed,chambolle2022accelerated}.
Stochastic and continuous barycenter methods have also been developed
\citep{claici2018stochastic,li2020continuous}. The present work differs by keeping the entropic
regularization parameter fixed and using a semi-dual with a common-marginal constraint. Its finite-sum
components correspond to input distributions and their support points, which gives an accelerated
variance-reduced stochastic method.

For large-scale entropic OT, the standard computational reference is the Sinkhorn algorithm
\citep{cuturi2013sinkhorn}. Sinkhorn and its variants, including stochastic Sinkhorn
\citep{abid2018stochastic}, are based on Bregman projections \citep{benamou2015iterative}. With
regularization tuning and rounding, Sinkhorn and related schemes can approximate the unregularized OT problem
with complexity $\widetilde{\mathcal{O}}\!\left(n^2/\epsilon^2\right)$
\citep{dvurechensky2018computational}. These guarantees address a different target from ours. We work
with the entropic objective at a fixed regularization parameter. We therefore report Sinkhorn in the
experiments as a standard computational reference; the complexity comparison uses fixed-$\eta$
accelerated gradient.

PDASGD belongs to the family of first-order primal-dual methods for linearly constrained convex
optimization. Related accelerated algorithms include APDAGD \citep{dvurechensky2018computational},
APDAMD \citep{lin2019efficient}, AAM \citep{guminov2021combination}, APDRCD \citep{guo2020fast}, and
HPD \citep{chambolle2022accelerated}. For the fixed-regularization entropic OT specialization, these
deterministic accelerated primal-dual methods have complexity
$\widetilde{\mathcal{O}}_\eta\!\left(n^{2.5}/\sqrt{\epsilon}\right)$, while
Section~\ref{applytoot} shows that PDASGD attains
$\widetilde{\mathcal{O}}_\eta\!\left(n^2/\sqrt{\epsilon}\right)$. Among stochastic primal-dual methods, PDASMD
\citep{luo2023improved} is closest to our two-marginal OT specialization. PDASMD uses mirror-descent
updates, whereas PDASGD uses Euclidean gradient steps with variance reduction. This distinction avoids
the norm-selection issue raised in \cite{luo2023improved}; in our analysis the improvement comes from
variance-reduced stochastic gradients and the same-order smoothness parameters of the dual.

Other approaches \citep{blanchet2024towards,jambulapati2019direct,quanrud2018approximating,lahn2019graph}
use packing LP or box-constrained Newton methods to obtain complexity
$\widetilde{\mathcal{O}}\!\left(n^2/\epsilon\right)$ for unregularized OT. As discussed in
\cite{lin2022efficiency}, these methods do not yet have the same level of large-scale computational
deployment as Sinkhorn and related first-order methods. PDASGD stays within the first-order setting, and the
experiments in Section~\ref{num} report its behavior on the entropic problem.
 
\subsection{Organization}
The remainder of the paper is organized as follows.
Section~\ref{prelim} introduces the two-marginal OT notation, entropic regularization, and smoothness conventions used throughout the paper.
Section~\ref{PDASGDintro} introduces PDASGD for linearly constrained optimization and states its convergence rate.
Section~\ref{improve} compares PDASGD with accelerated deterministic primal-dual first-order methods.
Section~\ref{applytoot} specializes PDASGD to two-marginal entropic OT and derives the resulting computational complexity.
Section~\ref{barycenter} develops PDASGD-BC for the entropic Wasserstein barycenter and proves its
$\widetilde{\mathcal O}_\eta\!\left(mn^2/\sqrt{\epsilon}\right)$ complexity.
Section~\ref{num} reports numerical experiments for entropic OT and barycenters.
Section~\ref{dis} discusses future directions.
Detailed proofs are provided in the Supplementary Material.

\section{Preliminaries and Notation}\label{prelim}
\subsection{Two-Marginal OT and Entropic Regularization}\label{otfor}
We start with the standard two-marginal formulation. This fixes the notation for transport plans,
marginals, costs, and entropy used throughout the paper.
For two discrete probability distributions $\alpha$ and $\beta$, the optimal transport, or OT, problem is
\begin{equation}\label{introduction}
\begin{gathered}
\min _{X\in \mathcal{U}(\alpha,\beta)}  \langle C,X\rangle,\\
\mathcal{U}(\alpha,\beta)=\left\{X\in \mathbb {R}^{n\times n}_{+}\left| X \textbf{1}_n=\alpha, X^{\top}\textbf{1}_n=\beta\right.\right\},
\end{gathered}
\end{equation}
where $\alpha, \beta \ge 0$ and $\alpha^{\top} \textbf{1}_n=\beta^{\top} \textbf{1}_n=1$.
$X$ and $C$ denote the transport plan and the cost matrix, respectively.
The matrix inner product is defined as $\langle C,X\rangle = \sum_{i,j=1}^{n}C_{ij}X_{ij}$.

\begin{remark}
For presentation, we restrict attention to the case in which both marginals have the same dimension $n$. The square case already captures the core difficulty of the entropic OT problem, and our analysis of PDASGD focuses on this regime.
\end{remark}

Entropic OT \citep{cuturi2013sinkhorn} adds an entropy regularizer to make the problem smooth and
differentiable with respect to the marginal distribution. The regularized problem is
\begin{equation}
\label{eq:entropic-ot}
\min_{X\in \mathcal{U}(\alpha,\beta)}\langle C,X\rangle-\eta H(X),
\end{equation}
where $\eta$ is the regularization parameter and $H(X)$ is the entropy.

\subsection{Accuracy Convention}\label{accuracy}
Throughout the paper, $\epsilon>0$ denotes an objective-gap target for the objective under discussion.
For a deterministic output $\widehat x$ to a problem $\min_x F(x)$ with optimal value $F^\star$,
$\epsilon$-accuracy means $F(\widehat x)-F^\star\le\epsilon$. For a randomized output, it means
$\mathbb E[F(\widehat x)]-F^\star\le\epsilon$. Thus, in unregularized comparisons, $\epsilon$ refers
to the OT objective in \eqref{introduction}; in our fixed-$\eta$ results, it refers to the entropic OT
objective in \eqref{eq:entropic-ot} or to the entropic barycenter objective in \eqref{bc:primal}.

\subsection{Notation and Smoothness Definitions}\label{notation}
We use the following notations.

The $m$-dimensional column of all ones is denoted by $\textbf{1}_{m}$.
The $i$th standard basis vector in $\mathbb R^m$ is denoted by $\textbf{e}_i$.
The $m\times m$ identity matrix is denoted by $\text{I}_m$.
For vectors, $\Vert\cdot\Vert_p$ denotes the $\ell_p$ norm.
$X\otimes Y$ denotes the standard Kronecker product of matrices $X$ and $Y$.
The entropy of matrix $X$ is defined as $H(X) = -\sum X_{ij} \ln(X_{ij})$.
$\Vert X\Vert_1=\sum\vert X_{ij}\vert$ and $\Vert X\Vert_\infty=\sup_{i,j} \vert X_{ij}\vert$.
For a matrix $X$, $\exp(X)$ and $\ln(X)$ are applied elementwise.
For $X\in\mathbb{R}^{n\times m}$, the column-major vectorization is $\text{Vec}(X)=(X_{11},\cdots,X_{n1},\cdots,X_{1m},\cdots,X_{nm})^{\top}$.
The matrix norm induced by vector norms $\|\cdot\|_H$ and $\|\cdot\|_E$ is denoted by $\|X\|_{E,H}= \max_{\|a\|_E\leq 1} \|Xa\|_H$.
Asymptotic notation is used as follows. The notation $A=\mathcal O(B)$ means
that $A\le cB$ for a numerical constant $c$ independent of the problem dimensions and
the accuracy target, and $A=\Theta(B)$ means that both $A=\mathcal O(B)$ and
$B=\mathcal O(A)$ hold. In particular, $d=\Theta(n)$ means that there exist positive
constants $c_1$, $c_2$, and $n_0$ such that $c_1n \le d \le c_2n$ for all $n\ge n_0$.
The condition $r\ge n^{-\mathcal O(1)}$ means that $r\ge c n^{-p}$ for some constants
$c>0$ and $p\ge0$ independent of $n$. The notation $\widetilde{\mathcal O}$ suppresses
polylogarithmic factors in the dimensions and in the inverse accuracy. When the
regularization parameter is fixed, $\widetilde{\mathcal O}_\eta$ may additionally hide
constants depending on $\eta$ and logarithmic factors induced by the polynomial mass lower bounds,
but it does not hide dependence on the cost scale, $m$, $n$, or $1/\epsilon$.

\begin{definition}
A continuously differentiable and convex function $f:Q \rightarrow \mathbb{R}$, where $Q$ is a convex and closed subset of $\mathbb{R}^q$, is said to be
\begin{itemize}
\item   $L$-smooth w.r.t. $\Vert \cdot \Vert_2$ in $Q$ if
$\Vert \nabla f(x)-\nabla f(y)\Vert_2\leq L\Vert x-y\Vert_2, \forall x,y\in Q.$ Or equivalently,
$f(y)\leq f(x) + \langle \nabla f(x),y-x\rangle + L/2\|x-y\|^2_2, \forall x,y\in Q.$
\item   $L$-smooth w.r.t. $\Vert \cdot \Vert_\infty$ in $Q$ if $\Vert \nabla f(x)-\nabla f(y)\Vert_1\leq L\Vert x-y\Vert_\infty, \forall x,y\in Q.$
Or equivalently,
$f(y)\leq f(x) + \langle \nabla f(x),y-x\rangle + L/2\|x-y\|^2_\infty, \forall x,y\in Q.$
\item $\sigma$-strongly convex w.r.t. $\Vert\cdot\Vert_E$ in $Q$ if $f(y)\geq f(x)+\langle \nabla f(x),y-x\rangle+\sigma/2\Vert x-y\Vert_E^2, \forall x,y\in Q.$
\end{itemize}
\end{definition}

\section{The PDASGD Algorithm}
\label{PDASGDintro}
The entropic OT problem \eqref{eq:entropic-ot} is an optimization problem with the linear constraints $X\in\mathcal{U}(\alpha,\beta)$. 
We first describe a linearly constrained optimization problem that covers entropic OT as a special case,
and then present the accelerated stochastic primal-dual algorithm.
\subsection{Linearly Constrained Optimization Problem}\label{general}
We present the general linearly constrained problem, its dual, and the assumptions used in the
convergence analysis.

Consider the linearly constrained optimization problem
   \begin{equation}
   \begin{aligned}
   \label{generalprimalmain}
        \min_{x\in Q\subset\mathbb{R}^q}\, f(x)\quad \\ \textrm{s.t.}\quad Ax=b\in\mathbb{R}^{d},
 \end{aligned}
   \end{equation}
   where $A\in\mathbb{R}^{d\times q}$, the feasible set is assumed to be nonempty, and $Q$ satisfies the conditions in the following assumption.
   
\begin{assumption}\label{assume0}
   In the optimization problem \eqref{generalprimalmain}, we assume that $f$ is real-valued, convex, coercive, and continuous, and $Q$ is closed and convex. 
\end{assumption}

   The Lagrange dual of \eqref{generalprimalmain} is
   \begin{equation}
       \begin{aligned}
       \label{generaldualmain}
   \min_{\lambda\in\mathbb{R}^{d}}\left\{\phi(\lambda)=-\langle b,\lambda\rangle+\max_{x\in Q\subset\mathbb{R}^q} (-f(x)+\langle Ax,\lambda\rangle)\right\}.
       \end{aligned}
   \end{equation}
   
   Let
     \begin{equation}
   \label{generalconvert}
   x(\lambda)=\arg\max_{x\in Q\subset\mathbb{R}^q} (-f(x)+\langle Ax, \lambda\rangle),
   \end{equation}
   Then the dual problem \eqref{generaldualmain} can be expressed as
   \begin{equation}
   \label{generaldualdualmain}
       \begin{aligned}
       \min_{\lambda\in\mathbb{R}^{d}}\left\{\phi(\lambda)=-\langle b,\lambda\rangle-f(x(\lambda))+\langle Ax(\lambda),\lambda\rangle\right\}.
       \end{aligned}
   \end{equation}
   
We next impose the structural properties used by the stochastic variance-reduction analysis.

\begin{assumption}
\label{assume} 
In the optimization problem \eqref{generalprimalmain} and its dual problem \eqref{generaldualmain}, we assume that 
\begin{itemize} 
    \item the following equation involving the gradient of $\phi(\lambda)$ and $x(\lambda)$ holds.
    \begin{equation}
\label{assumption}
    \nabla \phi (\lambda)= Ax(\lambda)-b.
\end{equation}
    \item  the function $\phi(\lambda)$ admits a finite-sum decomposition.
$$\phi(\lambda)=\frac{1}{h}\sum_{i=1}^{h}\phi_{i}(\lambda).$$
\item the function $\phi_i$ is convex and $L_i$-smooth w.r.t. $\Vert \cdot\Vert_2$ for each $i$. The average smoothness parameter is denoted by $$\overline{L}=\frac{1}{h}\sum_{i=1}^h L_i.$$
\end{itemize}   
\end{assumption}

Assumption~\ref{assume0} ensures existence of a primal solution to
\eqref{generalprimalmain}. In Assumption~\ref{assume}, the first condition is the primal-dual gradient
identity, the second gives the finite-sum structure needed for stochastic component gradients, and the
third supplies the component smoothness used in the variance-reduction analysis.

The next proposition records a sufficient condition for Assumption~\ref{assume} in the common case that $f$ is strongly convex.

\begin{proposition}\label{rem1}
Suppose $f$ is $\sigma$-strongly convex with respect to $\|\cdot\|_E$ on $Q$, and that $\phi(\lambda) = h^{-1}\sum_{i=1}^h \phi_i(\lambda)$ with each $\phi_i$ convex and differentiable. Then the following statements hold.
\begin{enumerate}
  \item Danskin's theorem \citep{bonnans2013perturbation} implies that $\nabla\phi(\lambda) = A x(\lambda) - b$ and that $\phi$ is convex;
  \item $\phi$ is $L$-smooth with respect to $\|\cdot\|_2$ with $L \le \|A\|_{E,2}^{2}/\sigma$ \citep{nesterov2005smooth};
  \item each $\phi_i$ is $hL$-smooth with respect to $\|\cdot\|_2$, and consequently the average smoothness satisfies $\overline{L} \le (h\,\|A\|_{E,2}^{2})/\sigma$.
\end{enumerate}
In particular, whenever $f$ is strongly convex, verifying Assumption~\ref{assume} reduces to verifying that $\phi$ admits a finite-sum decomposition with differentiable convex components.
\end{proposition}

The entropic OT objective is strongly convex with respect to $\Vert \cdot\Vert_1$
\citep{lin2019efficient,guminov2021combination}, which motivates applying this primal-dual analysis
to entropic OT and, later, to the barycenter problem.
\subsection{Our Algorithm}\label{algodetail}
We present \emph{primal-dual accelerated stochastic gradient descent}, abbreviated PDASGD, for
\eqref{generalprimalmain} and then analyze its convergence rate.

The pseudocode of PDASGD is presented in Algorithm \ref{PDASGD}.
It is a primal-dual algorithm motivated by \cite{allen2017katyusha}. 
 In the outer loop, Step 5 computes the full gradient. 
 In the inner loop,
 Step 8 uses Katyusha momentum \citep{allen2017katyusha} to accelerate the algorithm.
 Whereas a classical accelerated method averages $z$ and $y$, Katyusha momentum also includes the snapshot point $\widetilde{\lambda}$ in the extrapolated iterate.
 Steps 10 and 11 use the variance-reduced gradient computed in Step 9.
 The final averaging steps form the primal output as a weighted average of historical primal responses.
These recursions avoid storing all past values of $x$ and $\tau$.
Step 14 computes one primal response per outer iteration using a dual iterate sampled uniformly from
the preceding $M$ inner iterations.
\begin{algorithm}[hptb]
   \caption{PDASGD}
   \label{PDASGD} 
   \begin{algorithmic}[1]
\State {\bfseries Input:} {dual objective function $\phi(\lambda)$, smoothness parameters $L_i,\overline{L}$ w.r.t. $\Vert\cdot\Vert_2$, number of inner iterations $M$, and number of outer iterations $S$.}

   \State {$\tau_2\leftarrow 1/2$; $y_0=z_0=\widetilde{\lambda}^0=\lambda_0=C_0=D_0=s\leftarrow 0.$}
   \ForAll{$s=0,\cdots,S-1$}
   \State $\tau_{1,s}\leftarrow 2/(s+4)$; $\gamma_s\leftarrow 1/(9\tau_{1,s}\overline{L}).$
   \State $u^{s}\leftarrow \nabla \phi(\widetilde{\lambda}^s).$
   \ForAll{$j=0$ to $M-1$}
   \State $k\leftarrow sM+j.$
   \State $\lambda_{k+1}\leftarrow\tau_{1,s}z_k+\tau_2\widetilde{\lambda}^{s}+(1-\tau_{1,s}-\tau_2)y_k.$  
   \State $\widetilde{\nabla}_{k+1}\leftarrow u^{s}+(\nabla \phi_{i}(\lambda_{k+1})-\nabla \phi_{i}(\widetilde{\lambda}^{s}))/hp_i,$
   
   where $i$ is randomly chosen from $\{1,2,\cdots,h\}$, each with probability $p_i=L_i/(h\overline{L})$.
   \State $z_{k+1}\leftarrow z_{k}-
   \gamma_s\widetilde{\nabla}_{k+1}/2.$
   \State $y_{k+1}\leftarrow \lambda_{k+1}-\widetilde{\nabla}_{k+1}/(9\overline{L}).$ 
   \EndFor
   \State $\widetilde{\lambda}^{s+1}\leftarrow 1/M\sum_{j=1}^{M}y_{sM+j}.$
   \State $D_{s+1}\leftarrow D_s+x(\widehat{\lambda}_s)/\tau_{1,s},$ where $\widehat{\lambda}_s$ is uniformly randomly chosen from $\{\lambda_{sM+1},\cdots,\lambda_{sM+M} \}.$ 
   \State $C_{s+1}\leftarrow C_s+1/\tau_{1,s}.$
   \EndFor
   \State{\bfseries Output:} {$\widetilde{\lambda}^S,x^{S}=D_S/C_S.$}
   \end{algorithmic}
   \end{algorithm}
\subsection{Convergence Analysis}
We now state the convergence guarantees for Algorithm~\ref{PDASGD}.

The following theorem gives the convergence rate of PDASGD.
\begin{theorem}
\label{generaltheorem}
Under Assumptions~\ref{assume0} and~\ref{assume}, suppose the dual problem
\eqref{generaldualmain} admits a solution $\lambda^\ast$. If one applies PDASGD to
the two bounds \eqref{generalprimalmain} and \eqref{generaldualmain},
the output $x^{S}$ of Algorithm \ref{PDASGD}
satisfies
\begin{align*}
\mathbb{E}[f(x^S)] - f(x^{\ast})
&= \mathcal{O}\!\left(\frac{\phi(0)-\phi(\lambda^{\ast})}{S^2}+ \frac{\overline{L}\left\Vert \lambda^{\ast}\right\Vert_{2}^2}{MS^2}\right),\\
\left\Vert \mathbb{E}[ Ax^{S}-b]\right\Vert_1
&=\mathcal{O}\left(\frac{\phi(0)-\phi(\lambda^{\ast})}{S^2R}+ \frac{\overline{L}\,dR}{MS^2}\right)
\quad\text{for any }R\ge\|\lambda^\ast\|_\infty,
\end{align*}
where $x^{\ast}$ is a solution to problem~\eqref{generalprimalmain}, $\lambda^{\ast}$ is a solution to problem~\eqref{generaldualmain}, $d$ is the dimension of the dual variable, and $\widetilde{\lambda}^{S}$ denotes the final value of the outer-loop iterate produced in Step~13 of PDASGD after $S$ outer iterations. In applications, $R$ is chosen as a positive upper bound on $\|\lambda^\ast\|_\infty$.
\end{theorem}

If the dual objective is additionally $L^{\prime}$-smooth with respect to
$\Vert\cdot\Vert_\infty$, the convergence bound simplifies as follows. This form clarifies the
conditions under which PDASGD attains a strictly better arithmetic bound, as developed in
Section~\ref{improve}.
\begin{corollary}\label{theorem1cor}
Under Assumptions~\ref{assume0} and~\ref{assume}, suppose the dual problem
\eqref{generaldualmain} admits a solution $\lambda^\ast$ and $\phi(\lambda)$ is $L^{\prime}$-smooth
with respect to $\Vert\cdot\Vert_\infty$. If one applies PDASGD to
the two bounds \eqref{generalprimalmain} and \eqref{generaldualmain},
 the output $x^{S}$ of Algorithm \ref{PDASGD}
satisfies
\[
\mathbb{E}[f(x^S)] -f(x^{\ast})=\mathcal{O}\left(\frac{L^\prime \left\Vert\lambda^\ast\right\Vert_\infty^2}{S^2}+ \frac{\overline{L}\left\Vert \lambda^{\ast}\right\Vert_{2}^2}{MS^2}\right), \]
\[ \left\Vert\mathbb{E}[Ax^{S}-b]\right\Vert_1=\mathcal{O}\left(\frac{L^\prime \left\Vert\lambda^\ast\right\Vert_\infty^2}{S^2R}+ \frac{\overline{L}\,dR}{MS^2}\right)\quad\text{for any }R\ge\|\lambda^\ast\|_\infty,\]
where $x^{\ast}$ is a solution to problem \eqref{generalprimalmain}, $\lambda^{\ast}$ is a solution to problem \eqref{generaldualmain}, and $d$ is the dual dimension.
\end{corollary}
\begin{remark}
If the inner-loop length $M$ equals $d$, the inequality $\left\Vert\lambda^\ast\right\Vert_2\leq \sqrt{d}\left\Vert\lambda^\ast\right\Vert_\infty$ gives
\[
\mathbb{E}[f(x^S)] -f(x^{\ast})=\mathcal{O}\left(\frac{(L^\prime+\overline{L})\left\Vert\lambda^\ast\right\Vert_\infty^2}{S^2}\right).\]

Thus the computational gain depends on the relative orders of $L^{\prime}$ and $\overline{L}$. For the
entropic-OT semi-dual, both quantities are constants independent of $n$, which is the mechanism behind
the improved complexity developed in Sections~\ref{improve} and~\ref{applytoot}.
\end{remark}
\section{When PDASGD Improves the Computational Complexity}\label{improve}
We state the scaling conditions under which PDASGD reduces the total arithmetic complexity of reaching
a prescribed objective gap by a factor of $\sqrt{n}$ compared with accelerated deterministic
primal-dual first-order methods.

We focus on the primal problem \eqref{generalprimalmain} and the associated dual problem \eqref{generaldualmain} introduced in Section \ref{general}.
   
We impose one additional scaling assumption to make the comparison transparent.

\begin{assumption}\label{assume2}
In problem \eqref{generaldualmain}, slightly abusing notation, assume that
\begin{itemize}
 \item  $d=\Theta(n)$.
\item $\phi(\lambda)$ admits the finite-sum decomposition
$\phi(\lambda)=(1/n)\sum_{i=1}^n\phi_i(\lambda)$.
\item $\phi(\lambda)$ is $\widetilde{L}$-smooth with respect to $\|\cdot\|_2$.
\item Each component $\phi_i(\lambda)$ is $L_i$-smooth with respect to $\|\cdot\|_2$, with average
smoothness $\overline{L}=(1/n)\sum_{i=1}^n L_i$.
\item $\phi(\lambda)$ is $L'$-smooth with respect to $\|\cdot\|_\infty$.
\item \emph{Same-order smoothness condition.} There exist constants $c_1,c_2>0$ and an exponent $a \ge 0$, all independent of $n$, such that
\[
c_1\, n^{a} \;\le\; \min\{\widetilde{L},\,\overline{L},\,L'\} \;\le\; \max\{\widetilde{L},\,\overline{L},\,L'\} \;\le\; c_2\, n^{a}.
\]
Equivalently, $\widetilde{L}=\Theta(n^{a})$, $\overline{L}=\Theta(n^{a})$, and $L'=\Theta(n^{a})$ with the same exponent $a$.
\end{itemize}
\end{assumption}

For comparison, we restate the convergence rate of APDAGD \citep{dvurechensky2018computational}.
\begin{theorem}
    [Theorem 3 in \cite{dvurechensky2018computational}]\label{apdagd}
 Suppose $f$ in problem \eqref{generalprimalmain} is strongly convex on the closed convex set $Q$, and the dual objective function is $\widetilde{L}$-smooth w.r.t. the $\ell_2$ norm. If one applies APDAGD to \eqref{generalprimalmain} and \eqref{generaldualmain}, the output $x^{k}$ of the kth iteration generated by APDAGD satisfies
 \[f(x^{k})-f(x^{\ast})=\mathcal{O}\left(  \frac{\widetilde{L}\left\Vert \lambda^{\ast}\right\Vert_2^2}{k^2}\right).\]
\end{theorem}

Using the APDAGD convergence rate, Corollary~\ref{coro} compares APDAGD and PDASGD. PDASGD computes the gradient of one component function, $\phi_i(\lambda)$, at each iteration, whereas APDAGD computes the full gradient over all $n$ components.
For clarity, we additionally assume that $f$ is strongly convex, ensuring smoothness of the dual
objective $\phi(\lambda)$ as discussed in Proposition~\ref{rem1}.
\begin{corollary}\label{coro}
Suppose that the computational complexity per iteration is $\mathcal{O}(\mathcal{K})$ for PDASGD, where only one gradient component is computed, and $\mathcal{O}(n\mathcal{K})$ for APDAGD, where the full gradient is computed. Additionally, assume that $f$ is strongly convex. Under Assumptions~\ref{assume0}, \ref{assume}, and~\ref{assume2}, applying PDASGD with $M=n$ inner loops and APDAGD to \eqref{generalprimalmain} and \eqref{generaldualmain} gives total computational complexity  $$\mathcal{O}\left(n^{\frac{3}{2}}\mathcal{K}\sqrt{\frac{L^{\prime}}{\epsilon}}\left\Vert \lambda^\ast\right\Vert_\infty\right)$$ for APDAGD and  $$\mathcal{O}\left(n\mathcal{K}\sqrt{\frac{L^{\prime}}{\epsilon}}\left\Vert \lambda^\ast\right\Vert_\infty\right)$$ for PDASGD, to obtain a solution $\vec{x}$ satisfying $f(\vec{x})-f(x^{\ast})\leq \epsilon$, or $\mathbb{E}[f(\vec{x})]-f(x^{\ast})\leq \epsilon$ for a random output.
\end{corollary}

Corollary~\ref{coro} shows that PDASGD reduces the total arithmetic complexity of producing a primal
output with objective gap at most $\epsilon$ by a factor of $\sqrt{n}$ relative to APDAGD.
Other accelerated primal-dual first-order algorithms have convergence rates or computational complexities comparable to APDAGD, including APDAMD \citep{lin2019efficient}, APDRCD \citep{guo2020fast}, and AAM \citep{guminov2021combination}. We omit these parallel comparisons to conserve space.

Thus, Corollary~\ref{coro} identifies the conditions under which PDASGD has lower arithmetic
complexity than accelerated deterministic primal-dual first-order methods for problems of the
form of \eqref{generalprimalmain} and \eqref{generaldualmain}. The requirement is that the smoothness
parameters $\widetilde{L}, \overline{L}, L^{\prime}$ in Assumption~\ref{assume2} have the same order.
Section~\ref{applytoot} verifies this condition for the semi-dual of the entropic OT problem.

\section{Two-Marginal Entropic OT Specialization}\label{applytoot}
We next specialize the general PDASGD analysis to fixed-regularization two-marginal entropic OT. This
case serves both as a standalone application and as the basic transport subproblem behind the
barycenter analysis in Section~\ref{barycenter}. We derive the semi-dual, prove the primal gradient
identity, verify the smoothness assumptions used in Section~\ref{improve}, and obtain the
$\widetilde{\mathcal O}_\eta\!\left(n^2/\sqrt{\epsilon}\right)$ arithmetic complexity for the
entropic OT objective.

The entropic OT problem \citep{cuturi2013sinkhorn} is
 \begin{equation}\label{entropy}
 \min_{X\in \mathcal{U}(\alpha,\beta)}\langle C,X\rangle-\eta H(X),
\end{equation}
where $\eta$ is the regularization parameter and $H(X)$ is the entropy. It can be written in the
linearly constrained form of Section~\ref{general} as
\begin{equation}
\begin{aligned}
  	\label{stackmain}
	\min_{x\in\mathbb{R}^{n^2}_+}\{  f(x)&= c^{\top}x+\eta\, x^{\top}\ln x \}  \\ &\textrm{s.t.}\quad Ax=b,
	\end{aligned}
\end{equation}
where $c=\text{Vec}(C)$, $x=\text{Vec}(X)$, $b=(\alpha^{\top},\beta^{\top})^{\top}$, and $A\in\mathbb{R}^{2n\times n^2}$ is the linear operator such that
$Ax=A\text{Vec}(X)= \begin{pmatrix}
 X\textbf{1}_n\\
 X^{\top}\textbf{1}_n
\end{pmatrix}$.  Under the column-major vectorization convention, the operator is
$$A = \left(\begin{array}{c}\textbf{1}_n^\top\otimes \text{I}_n\\\text{I}_n \otimes \textbf{1}_n^\top \end{array}\right).$$

Solving the primal response problem \eqref{generalconvert} gives
\begin{equation}\label{xlambda}
x(\lambda)=\exp\left(\frac{A^{\top}\lambda-c-\eta\textbf{1}_{n^2}}{\eta}\right),  \quad
\left[A^\top\lambda\right]_{i+n(j-1)}=\lambda_i+\lambda_{n+j},\quad 1\le i,j\le n,
\end{equation}
where $\left[A^\top\lambda\right]_{i+n(j-1)}$ follows the column-major vectorization convention for the
entry $X_{ij}$.

Substituting \eqref{xlambda} into \eqref{generaldualdualmain} gives the following dual of
\eqref{stackmain}.
\begin{equation}
\label{section3dual}
    \min_{\lambda\in \mathbb{R}^{2n}}\left\{-\langle b,\lambda\rangle
    +\eta\left\langle \exp\left(\frac{A^{\top}\lambda-c-\eta\textbf{1}_{n^2}}{\eta}\right),
    \textbf{1}_{n^2}\right\rangle\right\},
\end{equation}

PDASGD requires a finite-sum dual objective. To obtain this structure, we follow
\cite{aude2016stochastic} and derive the semi-dual from \eqref{section3dual}. Split
$\lambda$ into $ \lambda = (u^{\top},v^{\top})^{\top}$, where $u,v\in \mathbb{R}^n$. Then
\eqref{section3dual} can be written as
\begin{equation}
\label{originaldual}
    \min_{u\in\mathbb{R}^{n},v\in\mathbb{R}^{n}}-\langle \alpha,u \rangle-\langle \beta,v \rangle+\eta\sum_{i=1}^{n}\sum_{j=1}^{n}\exp\left(\frac{u_i+v_j-c_{ij}-\eta}{\eta}\right).
\end{equation}

For a fixed $v$, minimizing \eqref{originaldual} with respect to $u$ gives
\begin{align*} u_i(v)=\eta\ln(\alpha_i)-\eta\ln\sum_{j=1}^{n}\exp\left(\frac{v_j-c_{ij}-\eta}{\eta}\right). 
 \end{align*}

Substituting $u(v)$ into \eqref{originaldual} gives
\begin{align*}
\begin{split}
    \min_{v\in\mathbb{R}^n}\left\{G(v)=\eta\sum_{i=1}^n\left(\alpha_i\ln\sum_{j=1}^{n}\exp\left(\frac{v_j-c_{ij}-\eta}{\eta}\right)\right)  -\sum_{j=1}^{n}\beta_jv_j-\eta\sum_{i=1}^n\alpha_i\ln\alpha_i+\eta\right\}.
\end{split}    
\end{align*}

Since $\sum_{i=1}^n\alpha_i=1$, we have the reformulations
\[\eta=\sum_{i=1}^n\eta\alpha_i, \quad\sum_{j=1}^n\beta_jv_j=\sum_{i=1}^n\alpha_i\left(\sum_{j=1}^n\beta_jv_j\right).\]

This gives the finite-sum semi-dual
\begin{align}
\begin{split}
\label{finitesumsemidual}
    \min_{v\in\mathbb{R}^n}\left\{G(v)=\frac{1}{n}\sum_{i=1}^{n} g_i(v)=\frac{1}{n}\sum_{i=1}^{n}n\alpha_ih_i(v)\right\},
\end{split}    
\end{align}
where
\begin{equation*}
h_i(v) = \eta\ln\sum_{j=1}^{n}\exp\left(\frac{v_j-c_{ij}-\eta}{\eta}\right)  -\sum_{j=1}^{n}\beta_jv_j-\eta\ln\alpha_i+\eta.
\end{equation*}

Problem \eqref{finitesumsemidual} is the semi-dual of the entropy-regularized OT problem. Since
$\lambda = (u^{\top},v^{\top})^{\top}=(u(v)^{\top},v^{\top})^{\top}$, \eqref{xlambda} gives the
primal response induced by $v$.
\begin{align*}
&x(v)=\exp\left(\frac{A^{\top}(u(v)^{\top},v^{\top})^{\top}-c-\eta\textbf{1}_{n^2}}{\eta}\right).  \label{relationxvmain} 
\end{align*}

  Hence $x(v)$ has the closed form
 \begin{equation}
 \label{closedformofprimalandsemidual}
 \begin{aligned}
     \left[x(v)\right]_{i+n(j-1)}&=\exp \left(\frac{u_i+v_j-c_{ij}-\eta}{\eta}\right)\\
     &=\exp\left(\frac{\eta\ln(\alpha_i)-\eta\ln\sum_{j=1}^{n}\exp\left(\frac{v_j-c_{ij}-\eta}{\eta}\right)+v_j-c_{ij}-\eta}{\eta}\right)\\
     &=\frac{\alpha_i\exp(\frac{v_j-c_{ij}-\eta}{\eta})}{\sum_{l=1}^{n}\exp(\frac{v_l-c_{il}-\eta}{\eta})}.
 \end{aligned}    
 \end{equation}
 
For the entropic OT problem \eqref{stackmain} and its semi-dual \eqref{finitesumsemidual},
Assumption~\ref{assume0} holds. To apply PDASGD, it remains to verify Assumption~\ref{assume}. The
gradient identity in \eqref{assumption} becomes the semi-dual identity in
Proposition~\ref{gradientequal}.
\begin{proposition}\label{gradientequal}
    With the notation above,
\[\left[Ax(v)-b\right]_{i=1,\cdots,n}=0,\]\[\nabla G(v)=\left[Ax(v)-b\right]_{i=n+1,\cdots,2n}, \]
where $\left[\cdot\right]_{i=1,\cdots,n}$ denotes the first $n$ coordinates and
$\left[\cdot\right]_{i=n+1,\cdots,2n}$ denotes the last $n$ coordinates.
\end{proposition}

For the finite-sum and smoothness parts of Assumption~\ref{assume}, it remains to verify that the
components $g_i$ in \eqref{finitesumsemidual} are convex and smooth. The next proposition gives the
needed constants.

\begin{proposition}
   \label{proposition1}
   In problem \eqref{finitesumsemidual}, 
   \begin{itemize}
       \item 
   $g_i$
   are convex.
   \item
   $g_i(v)$ is $n\alpha_i/\eta$-smooth w.r.t. $\Vert \cdot\Vert_2$ for all $1\leq i\leq n$. The average $L$-smooth parameter w.r.t. $\Vert\cdot\Vert_2$ of $G(v)$ is $1/\eta.$
   \item
   $G(v)$ is $5/\eta$-smooth w.r.t. $\Vert\cdot\Vert_{\infty}$.
   \end{itemize}
\end{proposition}
\begin{remark}
    Proposition~\ref{proposition1} shows that the semi-dual smoothness parameters relevant to
    PDASGD are independent of $n$ for fixed $\eta$. Together with the low component-gradient cost,
    this yields the complexity improvement stated in Section~\ref{improve}.
\end{remark}

The entropic OT semi-dual therefore satisfies the assumptions of Section~\ref{PDASGDintro}, so PDASGD
applies to \eqref{stackmain} through \eqref{finitesumsemidual}. We denote the vector-form output by
$\vec{x}\in\mathbb{R}^{n^2}$ and the corresponding matrix by
$\vec{X}\in\mathbb{R}^{n\times n}$, where $\vec{x}=\text{Vec}(\vec{X})$.

\begin{corollary}\label{coro2}
Assume $\min_i\alpha_i\ge n^{-O(1)}$ and
$\min_j\beta_j\ge n^{-O(1)}$. For fixed $\eta$, the total
arithmetic complexity of obtaining an $\epsilon$-accurate solution to the entropic OT problem
\eqref{stackmain} is
$\widetilde{\mathcal{O}}_\eta\!\left((n^{2.5}(1+\|C\|_\infty))/\sqrt{\epsilon}\right)$
for APDAGD and
$\widetilde{\mathcal{O}}_\eta\!\left((n^2(1+\|C\|_\infty))/\sqrt{\epsilon}\right)$ for PDASGD.
\end{corollary}

The semi-dual finite-sum structure gives PDASGD the improved complexity
$\widetilde{\mathcal{O}}_\eta\!\left(n^2/\sqrt{\epsilon}\right)$ for a fixed regularization parameter. The smoothness
parameter of the full dual with respect to $\Vert\cdot\Vert_2$ is $2/\eta$
\citep{dvurechensky2018computational,guo2020fast,lin2019efficient}, the same order as the semi-dual
parameters in Proposition~\ref{proposition1}. This same-order smoothness, together with low-cost
component gradients, yields the factor-$\sqrt n$ improvement.
\begin{remark}\label{rem:unreg}
We restrict attention to the entropically regularized OT problem. Rounding the PDASGD output to the
transport polytope $\mathcal{U}(\alpha,\beta)$ via the procedure of \cite{altschuler2017near} would be
the standard route to an $\epsilon$-approximation for unregularized OT, but the stochastic output introduces
an additional difficulty. The standard rounding analysis would have to control
$\mathbb{E}[\|A\vec{X}-b\|_1]$, whereas Theorem~\ref{generaltheorem} controls
$\|\mathbb{E}[A\vec{X}-b]\|_1$.
Converting between these two quantities requires a variance
bound on the primal iterate, which to our knowledge is not available for stochastic primal-dual
algorithms in this setting. A similar gap exists implicitly in \citealt{guo2020fast,luo2023improved}.
We therefore leave the unregularized extension to future work.
\end{remark}

\section{The Entropic Wasserstein Barycenter}\label{barycenter}
We now specialize the preceding analysis to the fixed-support entropic Wasserstein barycenter. The
inputs are distributions $\mu_1,\ldots,\mu_m\in\Delta_n$, weights $w_a>0$ with
$\sum_{a=1}^m w_a=1$, cost matrices $C_a\in\mathbb R^{n\times n}$, and a fixed regularization parameter
$\eta>0$. The unknown barycenter is the common target marginal. Each input distribution has its own
transport plan $\pi_a$, whose row marginal is fixed at $\mu_a$. The column marginals
$\pi_a^\top\mathbf1$ all equal the same vector $\nu$.

\subsection{Barycenter formulation}\label{bc:consensus}
The fixed-support entropic Wasserstein barycenter is obtained by solving
\begin{equation}\label{bc:primal}
\begin{aligned}
\min_{\{\pi_a\ge0\},\,\nu\in\mathbb R^n}\quad
&\sum_{a=1}^m w_a\Big(\langle C_a,\pi_a\rangle+\eta\langle\pi_a,\ln\pi_a\rangle\Big)\\
\text{s.t.}\quad
&\pi_a\mathbf1=\mu_a,\quad \pi_a^\top\mathbf1=\nu,\quad a=1,\ldots,m .
\end{aligned}
\end{equation}
The matrix $\pi_a\in\mathbb R^{n\times n}$ transports $\mu_a$ to the free barycenter marginal $\nu$. The
constraint $\nu\in\Delta_n$ is redundant because nonnegativity and $\pi_a\mathbf1=\mu_a\in\Delta_n$ imply
$\pi_a^\top\mathbf1\in\Delta_n$.

The variable $\nu$ only records the common column marginal. Equivalently, one may remove $\nu$ from
the formulation and require the column marginals $\pi_a^\top\mathbf1$ to be equal. With
$x=(\mathrm{vec}(\pi_1),\ldots,\mathrm{vec}(\pi_m))$, \eqref{bc:primal} is then an instance of
\eqref{generalprimalmain}. This constrained form is the starting point for the semi-dual derivation
below.

\subsection{Semi-dual structure}\label{bc:semidual}
Introduce row multipliers $\widehat f_a\in\mathbb R^n$ for $\pi_a\mathbf1=\mu_a$ and column multipliers
$\widehat g_a\in\mathbb R^n$ for $\pi_a^\top\mathbf1=\nu$. With the sign convention used in
Section~\ref{general}, the Lagrangian of \eqref{bc:primal} is
\[
\mathcal L
=\sum_{a=1}^m w_a\Big(\langle C_a,\pi_a\rangle+\eta\langle\pi_a,\ln\pi_a\rangle\Big)
-\sum_{a=1}^m\langle \widehat f_a,\pi_a\mathbf1-\mu_a\rangle
-\sum_{a=1}^m\langle \widehat g_a,\pi_a^\top\mathbf1-\nu\rangle .
\]
Since $\nu$ is free, minimizing over $\nu$ is finite only when
\[
\sum_{a=1}^m\widehat g_a=0.
\]
Set $f_a=\widehat f_a/w_a$ and $g_a=\widehat g_a/w_a$. The condition above becomes the weighted
column-dual constraint
\[
H:=\left\{\mathbf g=(g_1,\ldots,g_m):\sum_{a=1}^m w_a g_a=0\right\}.
\]
In these scaled dual coordinates, the dual before eliminating the row variables is
\[
\min_{\mathbf g\in H,\,f_1,\ldots,f_m}
\sum_{a=1}^m w_a\left\{
-\langle \mu_a,f_a\rangle
+\eta\sum_{i=1}^n\sum_{j=1}^n
\exp\!\left(\frac{f_a^i+g_{a,j}-C_a^{ij}-\eta}{\eta}\right)
\right\}.
\]
For fixed $g_a$, minimizing with respect to the row variable $f_a$ gives
\[
f_a^i(g_a)
=\eta\ln\mu_{a,i}
-\eta\ln\sum_{j=1}^n
\exp\!\left(\frac{g_{a,j}-C_a^{ij}-\eta}{\eta}\right),
\qquad i=1,\ldots,n .
\]
Substituting $f_a(g_a)$ into the dual and dropping constants independent of $\mathbf g$ gives the
barycenter semi-dual on $H$,
\begin{equation}\label{bc:phi}
\Phi(\mathbf g)
=\sum_{a=1}^m w_a\sum_{i=1}^n\mu_{a,i}h_i^{(a)}(g_a)
=\frac{1}{mn}\sum_{a=1}^m\sum_{i=1}^n \Phi_{a,i}(\mathbf g),
\end{equation}
where
\[
h_i^{(a)}(g_a)
=\eta\ln\sum_{j=1}^n
\exp\!\left(\frac{g_{a,j}-C_a^{ij}-\eta}{\eta}\right),
\qquad
\Phi_{a,i}(\mathbf g)=mn\,w_a\mu_{a,i}h_i^{(a)}(g_a).
\]

The same substitution gives the primal response induced by $\mathbf g$. For support point $i$ of input
distribution $a$, define
\begin{equation}\label{bc:softmax}
p_{ij}^{(a)}(g_a)
=\frac{\exp\!\left(\frac{g_{a,j}-C_a^{ij}}{\eta}\right)}{\sum_{\ell=1}^n\exp\!\left(\frac{g_{a,\ell}-C_a^{i\ell}}{\eta}\right)},
\qquad j=1,\ldots,n .
\end{equation}
Then
\begin{equation}\label{bc:plan}
\pi_a(\mathbf g)^{ij}=\mu_{a,i}p_{ij}^{(a)}(g_a).
\end{equation}
Thus $p_i^{(a)}(g_a)\in\Delta_n$, and the recovered plan satisfies
$\pi_a(\mathbf g)\mathbf1=\mu_a$ by construction.

The component gradient of $\Phi_{a,i}$ has the explicit block form
\[
\nabla_{g_b}\Phi_{a,i}(\mathbf g)
=
\begin{cases}
mn\,w_a\mu_{a,i}\,p_i^{(a)}(g_a), & b=a,\\
0, & b\ne a.
\end{cases}
\]
Thus one stochastic component gradient requires one $n$-dimensional softmax, and its nonzero entries
are confined to the block $g_a$.

It remains to control whether the column marginals coincide. Define
\[
q_a(\mathbf g):=\pi_a(\mathbf g)^\top\mathbf1,
\qquad
\bar q(\mathbf g):=\sum_{b=1}^m w_bq_b(\mathbf g).
\]
The following identity is the barycenter analogue of Proposition~\ref{gradientequal}. Since the
semi-dual variable is constrained to $H$, the second line is stated through derivatives along
feasible directions.
An admissible direction is a vector $\delta=(\delta_1,\ldots,\delta_m)\in H$, so that
$\mathbf g+t\delta$ remains in $H$ for small $t$.

\begin{proposition}\label{bc:gradid}
With the notation above, for every $\mathbf g\in H$,
\[
\pi_a(\mathbf g)\mathbf1-\mu_a=0,\qquad a=1,\ldots,m,
\]
\[
\left.\frac{d}{dt}\Phi(\mathbf g+t\delta)\right|_{t=0}
=\sum_{a=1}^m\left\langle w_a\big(q_a(\mathbf g)-\bar q(\mathbf g)\big),\delta_a\right\rangle,
\qquad \delta\in H.
\]
\end{proposition}

Thus stationarity of $\Phi$ on $H$ means that this derivative is zero for every admissible direction
$\delta\in H$. By Proposition~\ref{bc:gradid}, this is equivalent to
$q_a(\mathbf g)=\bar q(\mathbf g)$ for every $a$, namely the recovered transport plans share a common
column marginal.

For the finite-sum and smoothness parts of Assumption~\ref{assume}, it remains to verify component
convexity and smoothness for \eqref{bc:phi}. The next proposition is the barycenter counterpart of
Proposition~\ref{proposition1}.
\begin{proposition}\label{bc:smooth}
In problem \eqref{bc:phi},
\begin{itemize}
\item the components $\Phi_{a,i}$ are convex;
\item $\Phi_{a,i}$ is $(mn\,w_a\mu_{a,i})/\eta$-smooth with respect to $\|\cdot\|_2$ for every
pair $(a,i)$, and the average $L$-smooth parameter is $\overline L=1/\eta$;
\item $\Phi$ is $5/\eta$-smooth with respect to $\|\cdot\|_\infty$ on $H$.
\end{itemize}
\end{proposition}
\begin{remark}
Proposition~\ref{bc:smooth} shows that the barycenter semi-dual has the same dimension-free
smoothness scale as the two-marginal semi-dual. Together with the fact that each component gradient is
one softmax over the barycenter support, this yields the improved support-size
dependence in Theorem~\ref{bc:main}.
\end{remark}

\subsection{PDASGD-BC and complexity}\label{bc:algo}
The constraint $\mathbf g\in H=\{\mathbf g:\sum_a w_ag_a=0\}$ is the dual feasibility condition obtained
when the barycenter marginal $\nu$ is eliminated. PDASGD-BC therefore runs Algorithm~\ref{PDASGD} on
this subspace. In block coordinates, this means that every full or stochastic gradient estimator is
projected onto $H$ before the inner-loop $z$- and $y$-updates. Since the iterates are initialized in $H$,
the affine updates and snapshot averages then remain in $H$.

Proposition~\ref{bc:smooth} gives
\[
N=mn,\qquad \overline L=\frac{1}{\eta},\qquad L'=\frac{5}{\eta}.
\]
The component indexed by $(a,i)$ differentiates only the block $g_a$ and requires the softmax
\eqref{bc:softmax}, so one component step costs $\mathcal O(n)$. A full-gradient snapshot evaluates all
$mn$ components and costs $\mathcal O(mn^2)$. The Euclidean projection onto $H$ has the explicit form
\[
(P_Hv)_a=v_a-w_a\frac{\sum_b w_bv_b}{\sum_b w_b^2},\qquad a=1,\ldots,m,
\]
and is applied to the gradient estimator. For single-component inner steps, the all-block correction is
kept as a lazy $n$-vector offset, so the projected component step remains $\mathcal O(n)$; full
materialization is needed only for snapshot gradients. We set the inner-loop length to $M=N$, so one
outer loop costs $\mathcal O(mn^2)$.

Combining these operation counts with Corollary~\ref{theorem1cor} gives the following complexity
bound.

\begin{theorem}[PDASGD-BC complexity]\label{bc:main}
Assume
$\min_{a,i}\mu_{a,i}\ge n^{-O(1)}$ and
$\min_j\nu^\star_j\ge n^{-O(1)}$, where $\nu^\star$ is the
entropic barycenter. For fixed $\eta$, PDASGD-BC returns a primal-dual output for \eqref{bc:primal}
with expected entropic barycenter objective residual at most $\epsilon$ in
\[
\widetilde{\mathcal O}_\eta\!\left(\frac{mn^2(1+\max_a\|C_a\|_\infty)}{\sqrt{\epsilon}}\right)
\]
arithmetic operations.
\end{theorem}

Theorem~\ref{bc:main} controls the barycenter objective residual, while
Proposition~\ref{bc:feas-main} below controls the remaining barycenter-feasibility residual. Equation~\eqref{bc:plan}
enforces the row marginals for every sampled primal response, and averaging preserves these
constraints. The residual below measures whether the recovered column marginals define a common
barycenter marginal, using their weighted average as the reference.

\begin{proposition}[Column-marginal residual of the averaged output]\label{bc:feas-main}
Let $x^S=(\pi_1^S,\ldots,\pi_m^S)$ be the averaged primal output of PDASGD-BC, and let
$\bar\nu^S:=\sum_{b=1}^m w_b(\pi_b^S)^\top\mathbf1$.
The row marginals satisfy $\pi_a^S\mathbf1=\mu_a$ exactly. Let $\widehat{\mathbf g}_s$ be the sampled
semi-dual point used to form the $s$th primal response, and set
$\delta_s=\mathbb E[\Phi(\widehat{\mathbf g}_s)-\Phi(\mathbf g^\star)]$. Then
\begin{equation}\label{bc:feas-bound}
\mathbb E\sum_{a=1}^m w_a\big\|(\pi_a^S)^\top\mathbf1-\bar\nu^S\big\|_1
\le
2\sqrt{\frac{10}{\eta\sum_{s=0}^{S-1}1/\tau_{1,s}}}
\sqrt{\sum_{s=0}^{S-1}\frac{\delta_s}{\tau_{1,s}}}.
\end{equation}
\end{proposition}

The residual on the left-hand side of \eqref{bc:feas-bound} is zero exactly when the recovered
transport plans share one column marginal. Thus the row constraints require no additional correction,
while the remaining barycenter-feasibility error decreases with the same semi-dual gaps used in the
objective analysis.

In the same fixed-$\eta$ entropic regime, deterministic accelerated gradient has rate
$\widetilde{\mathcal O}_\eta\!\left(mn^{2.5}/\sqrt{\epsilon}\right)$.
Theorem~\ref{bc:main} therefore improves the support-size dependence by a factor $\sqrt n$. For fixed
$m$, the $n^2$ dependence matches the two-marginal PDASGD specialization, although the guarantee now
concerns the barycenter objective.

\begin{table}[H]\centering
\caption{First-order complexity for the fixed-support Wasserstein barycenter of $m$ distributions of
dimension $n$. In each block, $\epsilon$ denotes the accuracy target for the objective named in that
block. The lower block is the fixed-$\eta$ entropic objective.}
\label{tab:bc-landscape}
{\small\setlength{\tabcolsep}{5pt}
\begin{tabular}{@{}p{0.52\textwidth}ll@{}}
\hline
Method & Complexity & Type \\
\hline
\multicolumn{3}{@{}l}{\textit{Unregularized $\epsilon$-accurate barycenter}}\\
\quad IBP, or Sinkhorn-barycenter \citep{benamou2015iterative,kroshnin2019complexity} & $\widetilde{\mathcal O}\!\left(mn^2/\epsilon^2\right)$ & deterministic \\
\quad FastIBP \citep{guminov2021combination} & $\widetilde{\mathcal O}\!\left(mn^{7/3}/\epsilon^{4/3}\right)$ & deterministic \\
\quad Area-convexity and dual extrapolation \citep{dvinskikh2021improved,lin2020fixed} & $\widetilde{\mathcal O}\!\left(mn^2/\epsilon\right)$ & deterministic \\
\hline
\multicolumn{3}{@{}l}{\textit{Entropic objective, fixed $\eta$}}\\
\quad Accelerated gradient \citep{kroshnin2019complexity,dvurechensky2018computational} & $\widetilde{\mathcal O}_\eta\!\left(mn^{2.5}/\sqrt{\epsilon}\right)$ & det.\ accel. \\
\quad \textbf{PDASGD-BC, this paper} & $\boldsymbol{\widetilde{\mathcal O}_\eta\!\left(mn^2/\sqrt{\epsilon}\right)}$ & \textbf{stoch.\ VR, accel.} \\
\hline
\end{tabular}}
\end{table}

\section{Numerical Experiments}\label{num}
We report numerical experiments for the two-marginal entropic OT specialization and for PDASGD-BC.
The two-marginal timing experiments were run on a 2023 16-inch MacBook Pro with 32 GB of RAM. The
larger barycenter experiments, including the DOTmark barycenter and the support-size scaling sweep,
were run on the Georgia Tech PACE Phoenix SLURM cluster. The scaling sweep used one-node array tasks with four
CPU cores and 4 GB of memory per core; the DOTmark image-barycenter run used one node with eight CPU
cores and 8 GB of memory per core. Because wall-clock times depend on the machine, the barycenter
scaling comparisons are reported primarily in full-gradient-equivalent work.

\subsection{Two-Marginal Entropic OT}\label{solveentropicot}
For the two-marginal tests in Figures~\ref{Fig:Data1} to \ref{Fig:Data4}, we use DOTmark images
\citep{schrieber2016dotmark} and the synthetic grayscale images of \cite{altschuler2017near}. The
synthetic images contain a randomly placed foreground square on a dark background. The foreground
occupies 20\% of the image; foreground and background intensities are sampled uniformly from
$[0,10]$ and $[0,1]$, respectively.

Each image is resized to $s\times s$, vectorized, and normalized to a marginal distribution of
dimension $n=s^2$. The cost matrix contains the $\ell_1$ distances between pixel locations. We use
five randomly selected image pairs for each data source and report error bars over these pairs. For
DOTmark, the source images are taken at $512\times512$ resolution before resizing. The regularizer is
$\eta=0.01$, and the stopping criterion is the primal-dual gap. For PDASGD, we set the inner-loop
length to $M=2\sqrt n$. The implementation uses the larger practical step size $15\gamma_s$ in
Step~10 of Algorithm~\ref{PDASGD}; the theory uses the conservative step size stated in the algorithm.

Figures~\ref{Fig:Data1} and \ref{Fig:Data2} use primal-dual gap tolerance $0.02$, while
Figures~\ref{Fig:Data3} and \ref{Fig:Data4} use tolerance $0.05$. Across the tested dimensions and
tolerances, PDASGD is faster than APDAGD, AAM, and PDASMD, the closest stochastic primal-dual
baseline. Sinkhorn is a specialized solver for the two-marginal entropic problem, so we report it as a
computational reference. The variability of
PDASGD across image pairs is small and comparable to the other first-order baselines.
\begin{figure}[htbp]
   \begin{minipage}{0.5\textwidth}
     \centering
     \includegraphics[width=0.8\linewidth]{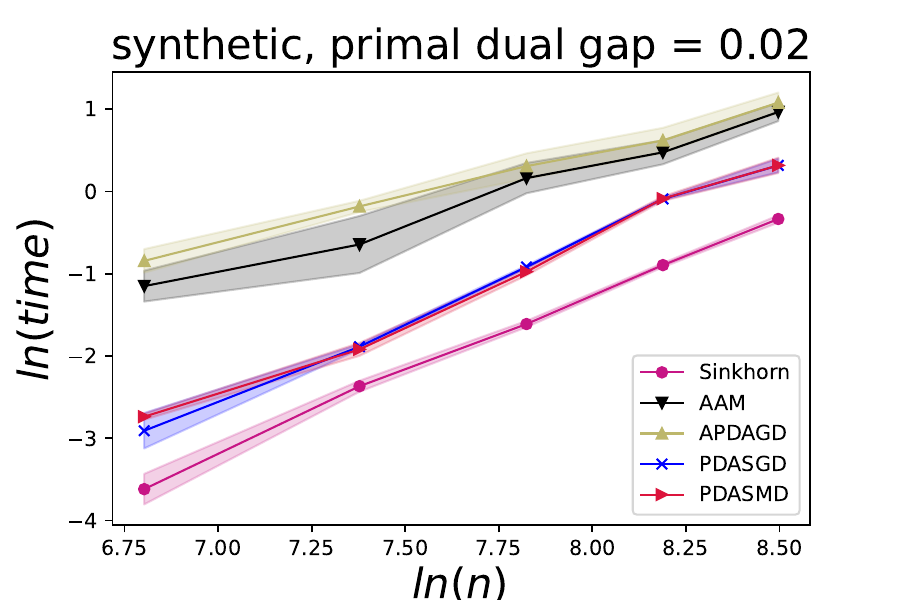}
     \caption{Entropic OT on synthetic data, gap tolerance $0.02$.}\label{Fig:Data1}
   \end{minipage}
   \begin{minipage}{0.5\textwidth}
     \centering
     \includegraphics[width=0.8\linewidth]{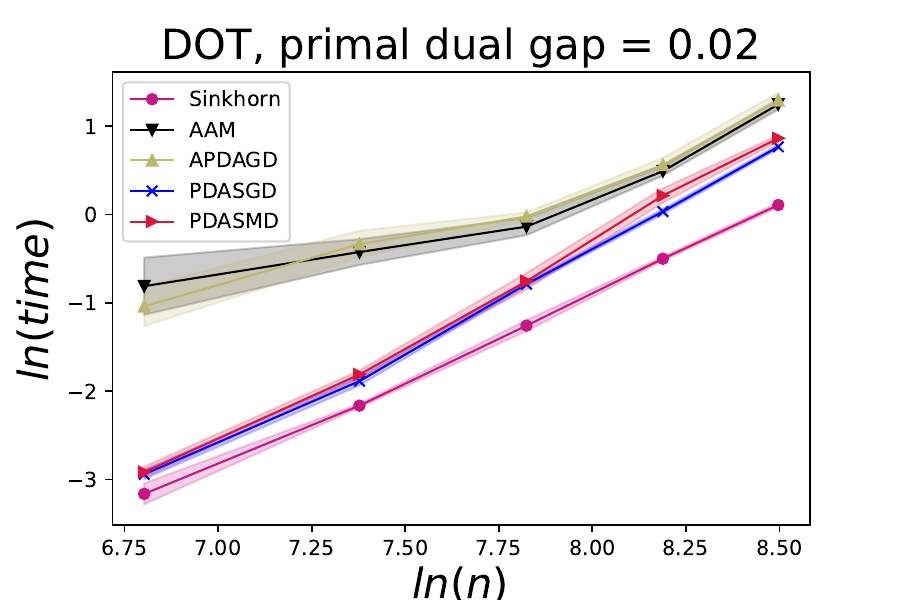}
     \caption{Entropic OT on DOTmark, gap tolerance $0.02$.}\label{Fig:Data2}
   \end{minipage}
\end{figure}
\begin{figure}[htbp]
   \begin{minipage}{0.5\textwidth}
     \centering
     \includegraphics[width=0.8\linewidth]{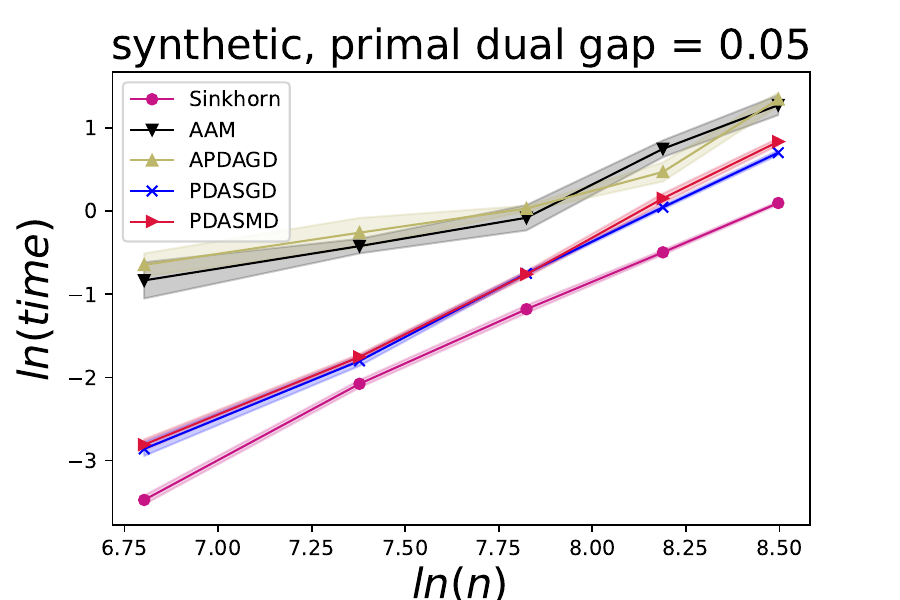}
     \caption{Entropic OT on synthetic data, gap tolerance $0.05$.}\label{Fig:Data3}
   \end{minipage}
   \begin{minipage}{0.5\textwidth}
     \centering
     \includegraphics[width=0.8\linewidth]{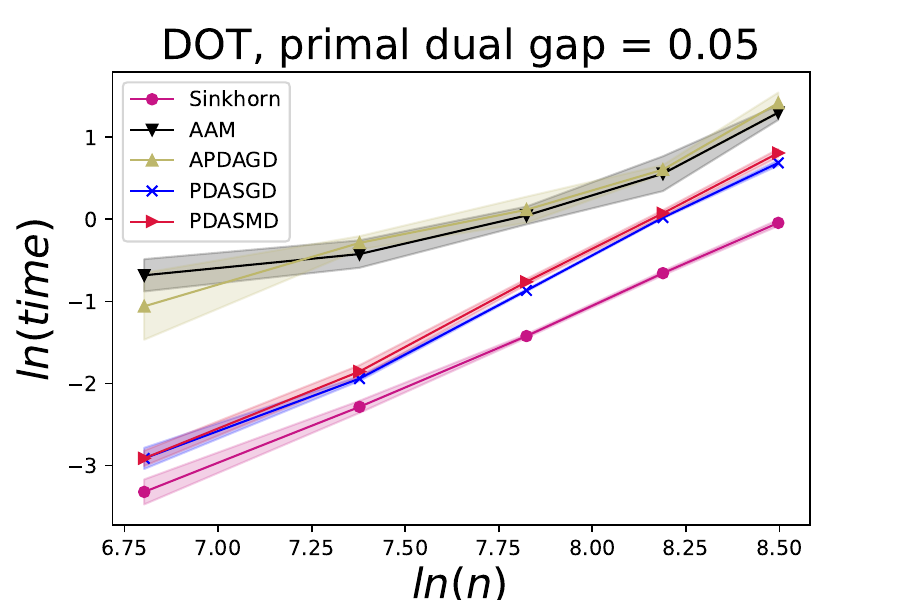}
     \caption{Entropic OT on DOTmark, gap tolerance $0.05$.}\label{Fig:Data4}
   \end{minipage}
\end{figure}

\subsection{Wasserstein Barycenter Experiments}\label{bc:experiments}
We compare PDASGD-BC with three baselines. IBP, or Sinkhorn-barycenter, is the standard deterministic
fixed-point method \citep{cuturi2014fast,benamou2015iterative}. APGD-BC is a deterministic
accelerated-gradient implementation for the same fixed-$\eta$ barycenter semi-dual, with the
$\widetilde{\mathcal O}_\eta\!\left(mn^{2.5}/\sqrt{\epsilon}\right)$ scaling in
Table~\ref{tab:bc-landscape}. SGD-BC
is a stochastic dual method without variance reduction, in the spirit of stochastic barycenter
methods \citep{claici2018stochastic,li2020continuous}. Together, these baselines distinguish
acceleration, stochastic sampling, and variance reduction. The FastIBP and area-convexity methods in
Table~\ref{tab:bc-landscape} are included for rate comparison but are not reimplemented in our experiments.

For these barycenter runs, the row marginals are enforced by construction, so we use the
column-marginal residual as the convergence diagnostic and report full-gradient-equivalent work as the
algorithmic cost proxy.

The Gaussian test uses three one-dimensional inputs, for which the unregularized Wasserstein
barycenter has a closed form \citep{agueh2011barycenters}. With $n=100$ and $\eta=0.004$, PDASGD-BC
recovers the closed-form barycenter mean of the unregularized model and differs from IBP by
$\|\nu_{\mathrm{PDASGD\text{-}BC}}-\nu_{\mathrm{IBP}}\|_1\approx1.3\times10^{-4}$; see
Figure~\ref{fig:bc-gauss}. A finite-difference check of the primal gradient identity in
Section~\ref{bc:semidual} gives a relative error of order $10^{-8}$. Figure~\ref{fig:bc-conv} then compares
PDASGD-BC, APGD-BC, and SGD-BC on the column-marginal residual, measured against
full-gradient-equivalent work. PDASGD-BC reaches about $10^{-5}$, while APGD-BC remains near $10^{-3}$
and SGD-BC near $10^{-2}$ on the same instance.

\begin{figure}[htbp]
  \begin{minipage}{0.52\textwidth}\centering
    \includegraphics[width=\linewidth]{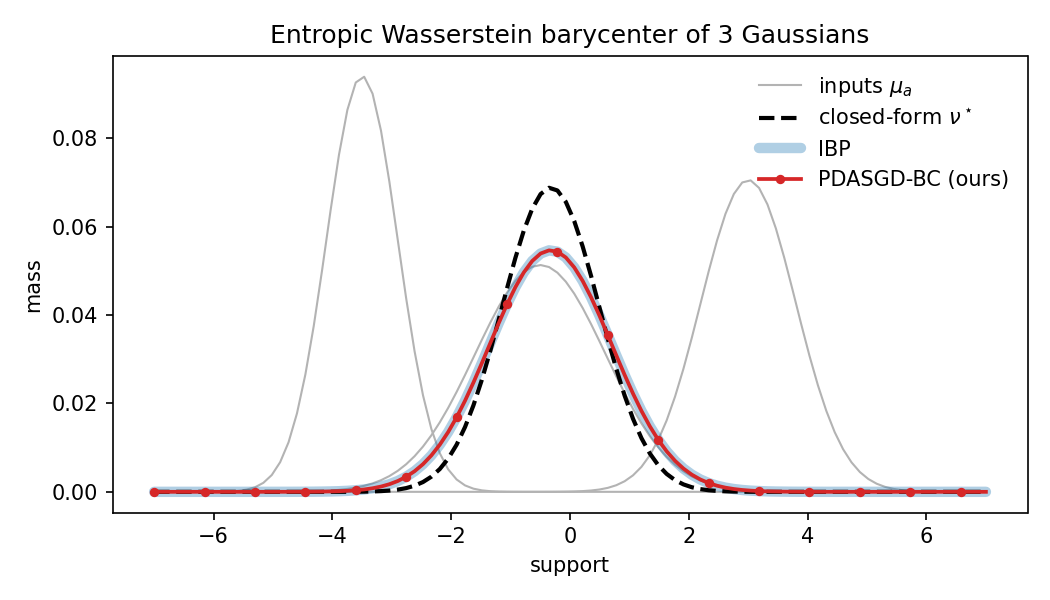}
    \caption{Three Gaussian inputs. PDASGD-BC matches the IBP reference and recovers the
    closed-form barycenter mean. The slight broadening relative to $\nu^\star$ is due to entropic
    regularization with $\eta=0.004$.}\label{fig:bc-gauss}
  \end{minipage}\hfill
  \begin{minipage}{0.46\textwidth}\centering
    \includegraphics[width=\linewidth]{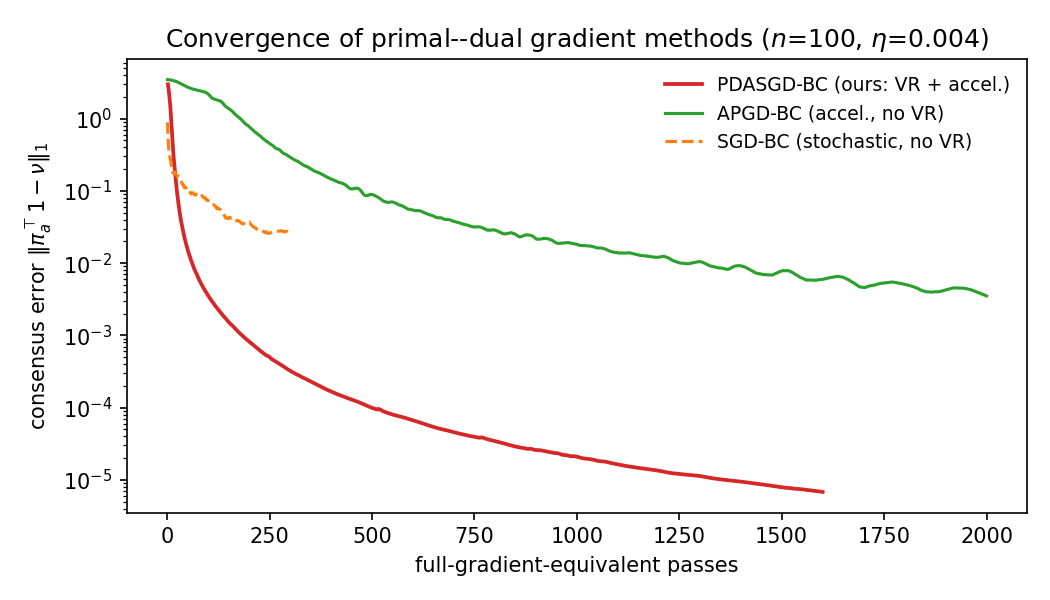}
    \caption{Column-marginal residual vs.\ work for $n=100$. PDASGD-BC combines acceleration and
    variance reduction; APGD-BC uses acceleration without variance reduction; SGD-BC is stochastic
    without variance reduction.}\label{fig:bc-conv}
  \end{minipage}
\end{figure}

Figures~\ref{fig:bc-shapes} and \ref{fig:bc-mnist} illustrate the recovered primal barycenters on
images. In Figure~\ref{fig:bc-shapes}, four binary shapes are represented as normalized histograms on
a $28\times28$ grid, so $n=784$. PDASGD-BC returns an equal-weight barycenter that matches the IBP
reference to $\|\cdot\|_1\approx1.7\times10^{-4}$. Sweeping the weights between two shapes gives the
entropic analogue of displacement interpolation \citep{mccann1997convexity}. Figure~\ref{fig:bc-mnist}
averages $12$ representative MNIST images of the digit $3$, selected from a larger candidate set
\citep{lecun1998gradient}. The Wasserstein barycenter aligns the strokes, whereas the Euclidean pixel
average blurs misaligned strokes.

\begin{figure}[htbp]\centering
  \includegraphics[width=0.92\linewidth]{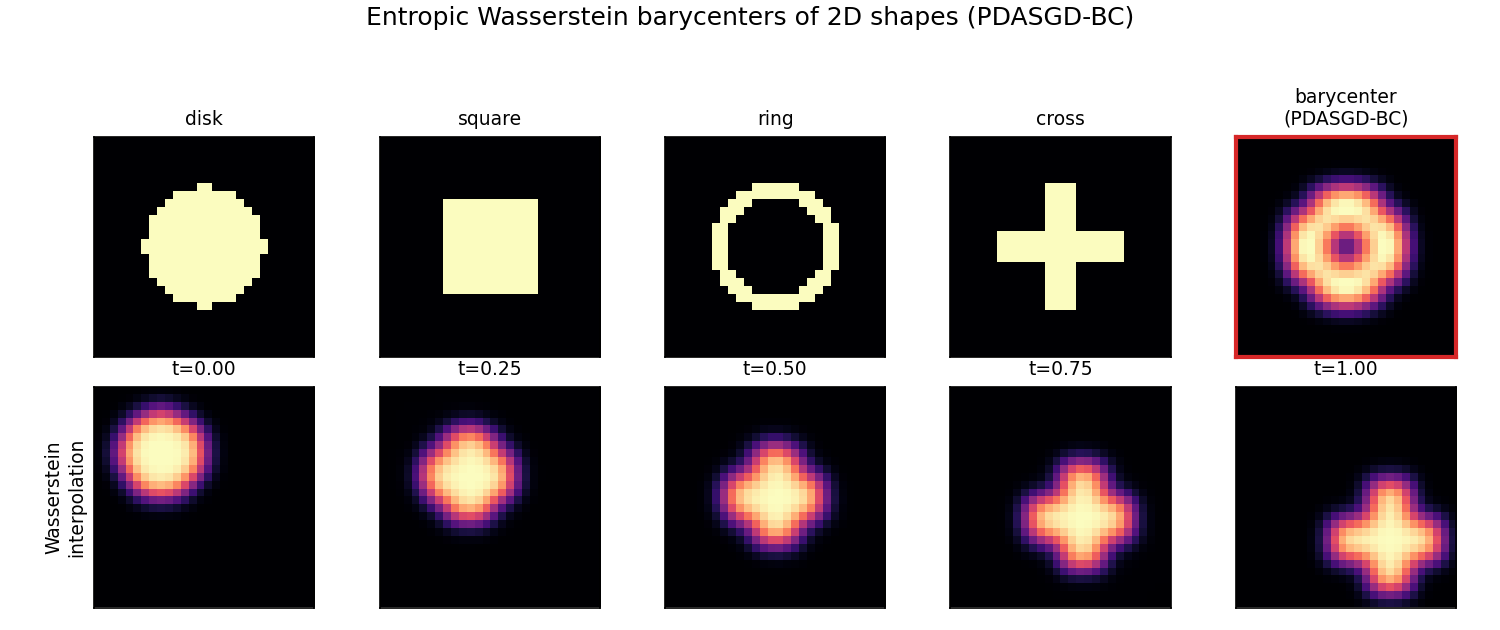}
  \caption{Shape barycenters with PDASGD-BC, with $n=28^2=784$. The top row shows four input shapes
  and their equal-weight entropic Wasserstein barycenter. The bottom row shows an interpolation
  obtained by sweeping the barycenter weight from $t=0$ to $t=1$.}\label{fig:bc-shapes}
\end{figure}

\begin{figure}[htbp]\centering
  \includegraphics[width=0.95\linewidth]{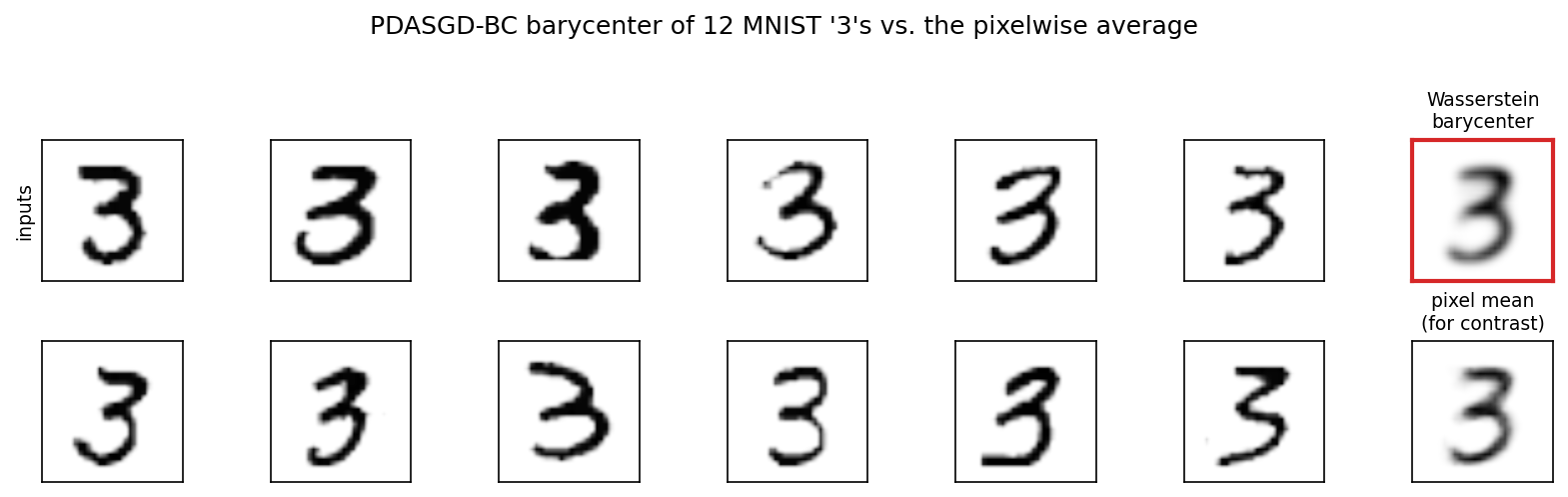}
  \caption{MNIST digit barycenter. The two left rows show $12$ handwritten instances of digit $3$. The
  right column compares the entropic Wasserstein barycenter, boxed in red, with the Euclidean
  pixelwise mean.}\label{fig:bc-mnist}
\end{figure}

We next test a larger image barycenter on DOTmark. Figure~\ref{fig:bc-dotmark} uses six images from
the \texttt{Shapes} class at $64\times64$ resolution, so $n=4096$. PDASGD-BC reaches a
column-marginal residual of $1.6\times10^{-5}$ and differs from IBP by
$\|\cdot\|_1\approx6.7\times10^{-4}$. In the recorded run, its wall-clock time is $93$ minutes,
compared with $206$ minutes for IBP.

\begin{figure}[htbp]\centering
  \includegraphics[width=\linewidth]{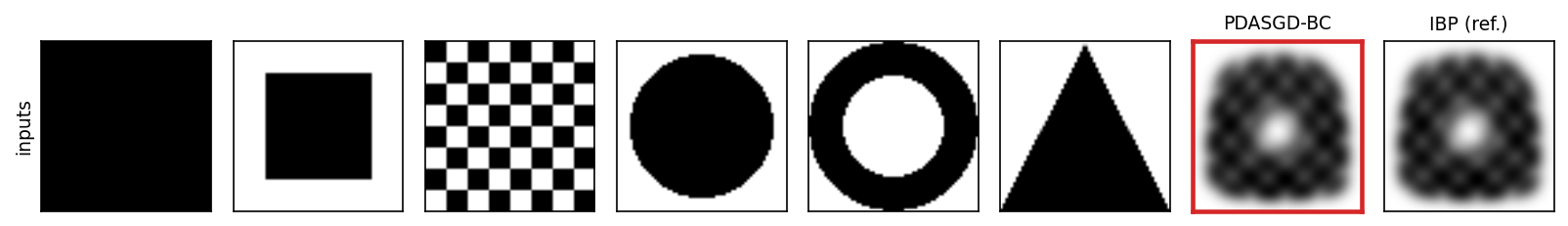}
  \caption{DOTmark barycenter at $64\times64$, with $n=4096$. The figure shows the six inputs, the
  PDASGD-BC barycenter in the red box, and the IBP reference. The two barycenters are visually
  indistinguishable, with $\|\cdot\|_1\approx6.7\times10^{-4}$.}\label{fig:bc-dotmark}
\end{figure}

Finally, Figures~\ref{fig:bc-scaling-pace} and \ref{fig:bc-scaling} examine the support-size scaling
suggested by Theorem~\ref{bc:main}, using a common residual tolerance as the stopping rule. In the
larger scaling run with six marginals, $\eta=0.02$, three seeds, and
$n\in\{60,120,240,480,720,960\}$, the fitted log-log slope of total arithmetic versus $n$ is $1.98$
for PDASGD-BC and $2.60$ for APGD-BC. These slopes are consistent with the predicted $n^2$ behavior of
PDASGD-BC and the approximately $\sqrt n$ separation from deterministic accelerated gradient. At
$n=960$, PDASGD-BC uses about $130$ full-gradient-equivalent passes, compared with about $27{,}000$
for APGD-BC. The smaller multi-seed run in Figure~\ref{fig:bc-scaling}(a) gives the same slopes
within sampling error. The residual-tolerance sweep in Figure~\ref{fig:bc-scaling}(b) has slope
$0.43$ against the inverse tolerance, close to the accelerated square-root dependence.

\begin{figure}[htbp]\centering
  \includegraphics[width=0.62\linewidth]{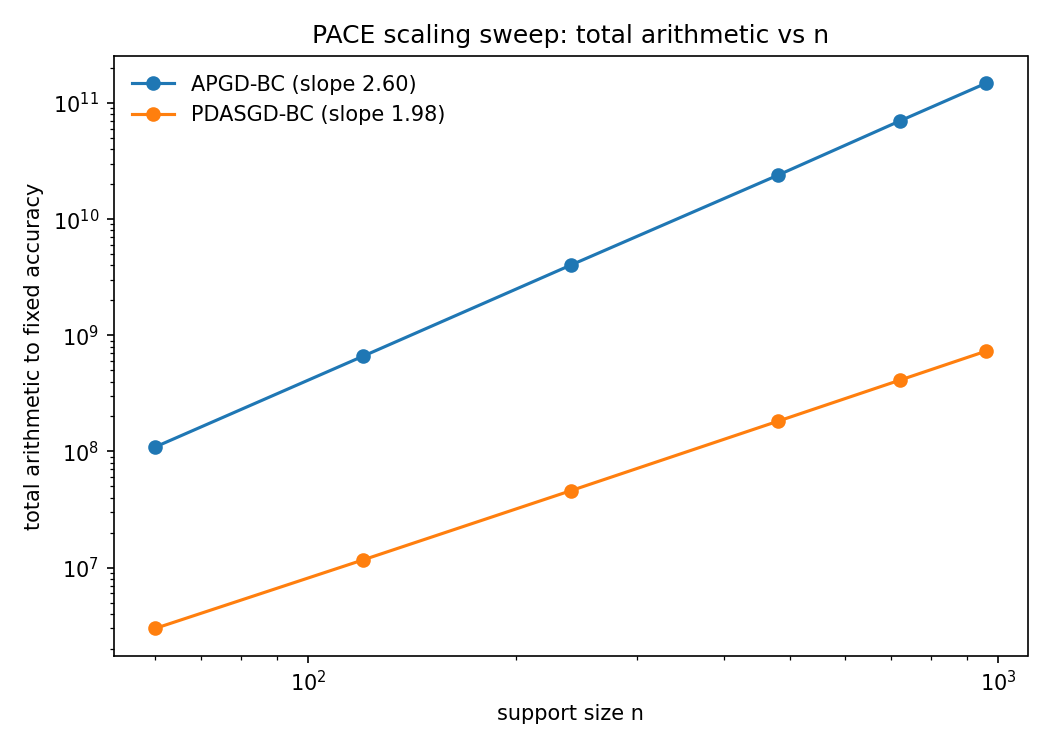}
  \caption{At-scale arithmetic to a fixed residual tolerance versus support size $n$. The fitted
  slopes are $1.98$ for PDASGD-BC and $2.60$ for APGD-BC, consistent with the $\sqrt n$ separation in
  Theorem~\ref{bc:main}.}\label{fig:bc-scaling-pace}
\end{figure}
\begin{figure}[htbp]\centering
  \includegraphics[width=0.92\linewidth]{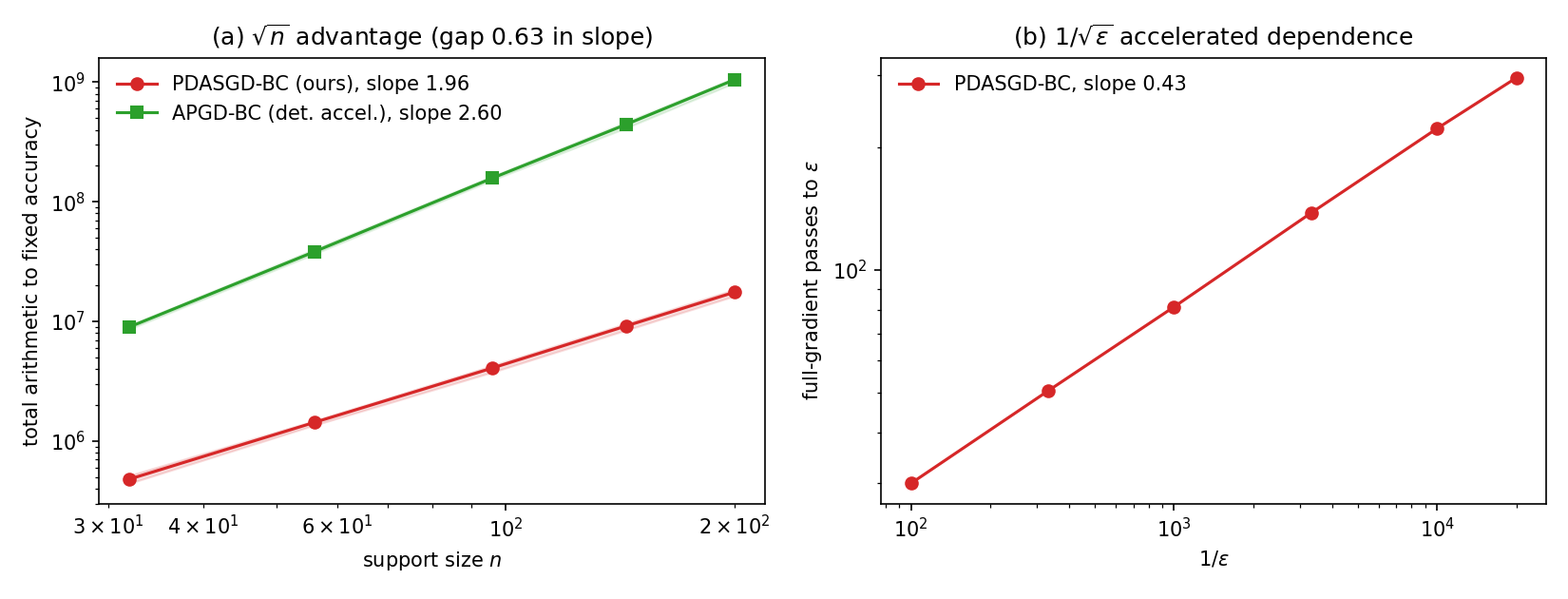}
  \caption{Multi-seed scaling and residual-tolerance dependence. Panel a reports arithmetic versus
  $n$ with seed-range bands. Panel b reports full-gradient-equivalent passes versus
  the inverse residual tolerance.}\label{fig:bc-scaling}
\end{figure}

\section{Discussion}\label{dis}
This paper develops PDASGD-BC, an accelerated stochastic variance-reduced algorithm for the
fixed-support entropic Wasserstein barycenter. The method uses a semi-dual with a linear
common-marginal constraint to compute a barycenter primal-dual output with expected entropic objective
residual at most $\epsilon$ in
$\widetilde{\mathcal O}_\eta\!\left(mn^2/\sqrt{\epsilon}\right)$ arithmetic operations, a
factor-$\sqrt n$
improvement over deterministic accelerated first-order methods in the same fixed-regularization regime.
The analysis also provides dimension-free smoothness bounds and a feasibility guarantee for the
recovered primal transport plans through an expected column-marginal residual bound. Experiments on
synthetic and image-based aggregation instances are consistent with the theoretical support-size and
accuracy behavior and include barycenter problems with many transport subproblems.

Fixed-$\eta$ entropic barycenters are widely used as regularized computational models. In imaging
and distributional aggregation, the regularization level is often chosen to control smoothness,
stability, and computational cost, and the resulting smoothed barycenter is the computed output
\citep{cuturi2013sinkhorn,cuturi2014fast,feydy2019interpolating}. The complexity bounds in this paper quantify the cost of solving this regularized model directly, while the
unregularized limit involves an additional regularization-tuning and rounding layer.

The fixed-regularization qualification is therefore substantive. The analysis exploits smoothness and finite-sum
separability; it does not rely on strong convexity of the semi-dual. If $\mu_\phi$ denotes the restricted
strong-convexity modulus of the entropic semi-dual, then in small-regularization regimes the softmax
weights can concentrate and $\mu_\phi$ can be exponentially small in $\|C\|_\infty/\eta$. Consequently,
arguments that convert Bregman divergence into squared Euclidean distance through $\mu_\phi$ would
carry exponentially poor constants. This explains why the paper separates fixed-$\eta$ entropic
complexity from unregularized complexity, and why Theorem~\ref{bc:main} does not claim a linear rate
or a rounded unregularized guarantee for the stochastic output. The improvement in
Theorem~\ref{bc:main} comes from the finite-sum components in \eqref{bc:phi} and their low evaluation
cost, not from a strongly convex reformulation in the number of input distributions.

Several directions remain. First, a variance bound on the primal iterate would enable a rigorous
rounding-based extension to the unregularized barycenter and to unregularized OT via
\cite{altschuler2017near}; see Remark~\ref{rem:unreg}. Second, a pathwise or high-probability theory
for stochastic primal-dual OT would require a mechanism that avoids relying on the poorly conditioned
semi-dual curvature. Third, extending the semi-dual construction to other structured multimarginal OT
problems, especially unbalanced and partial barycenters, is another direction.

\clearpage
\appendix

\section{Proofs for the General Results}\label{proofs}\label{prooftheorem2.5}
The proof uses the following auxiliary estimate, whose proof is given in Section~\ref{proofoffirstlemma}.
\subsection{Main Auxiliary Estimate}
\begin{lemma}\label{firstlemma}
    \begin{equation}
    \label{firstlemmacontent}
    \begin{aligned}
   &\frac{\tau_2M}{\tau_{1,S}^2}\left(\mathbb{E}[\phi(\widetilde{\lambda}^{S})]-\phi(\lambda^{\ast})\right)+\left(M\sum_{s=0}^{S-1}\frac{1}{\tau_{1,s}}\right) \left(\mathbb{E}[f(x^S)]-f(x(\lambda^{\ast}))\right)\\&\leq\left(M\sum_{s=0}^{S-1}\frac{1}{\tau_{1,s}}\right)\left\langle A\mathbb{E}[x^{S}]-b,\lambda\right\rangle +\frac{\tau_{2}M}{\tau_{1,0}^2}\left(\phi(\widetilde{\lambda}^0)-\phi(\lambda^{\ast})\right)\\
   &+\frac{1-\tau_{1,0}-\tau_2}{\tau_{1,0}^2} \left(\phi(y_{0})-\phi(\lambda^{\ast})\right)+9\overline{L}\left\Vert z_{0}-\lambda\right\Vert^2_2.
   \end{aligned}
\end{equation}
for all $\lambda\in\mathbb{R}^d$.
\end{lemma}
\subsection{Proof of Theorem~\ref{generaltheorem}}
\label{proofoftheorem3}
\begin{proof}
We prove the two bounds separately. Use the initialization $z_0=0$, $y_0=\widetilde{\lambda}^0=0$ and
$\tau_{1,0}=\tau_2=1/2$ (so $1-\tau_{1,0}-\tau_2=0$ and the $\phi(y_0)$ term of
\eqref{firstlemmacontent} vanishes), and abbreviate
\[
c_\phi:=\frac{\tau_2 M}{\tau_{1,S}^2},\qquad \Delta:=M\sum_{s=0}^{S-1}\frac{1}{\tau_{1,s}},\qquad
r:=A\,\mathbb{E}[x^S]-b,\qquad B_0:=\phi(0)-\phi(\lambda^{\ast}),\qquad f^{\ast}:=f(x(\lambda^{\ast})).
\]
Lemma~\ref{firstlemma} then states that, for \emph{every} $\lambda\in\mathbb{R}^d$,
\begin{equation}\label{master}
c_\phi\big(\mathbb{E}[\phi(\widetilde{\lambda}^{S})]-\phi(\lambda^{\ast})\big)
+\Delta\big(\mathbb{E}[f(x^S)]-f^{\ast}\big)
\ \le\ \Delta\langle r,\lambda\rangle+2M B_0+9\overline{L}\|\lambda\|_2^2 .
\end{equation}
We record three facts. First, by Jensen's inequality and optimality of $\lambda^{\ast}$,
$\mathbb{E}[\phi(\widetilde{\lambda}^{S})]\ge\phi(\mathbb{E}[\widetilde{\lambda}^{S}])\ge\phi(\lambda^{\ast})$,
so the first term on the left of \eqref{master} is nonnegative. Second, since $\nabla\phi(\lambda^{\ast})=0$
we have $\phi(\lambda^{\ast})=-f^{\ast}$, and, using the definition
$\phi(\lambda)=-\langle b,\lambda\rangle+\max_x\{-f(x)+\langle Ax,\lambda\rangle\}$ together with Jensen's
inequality, namely $\mathbb{E}[f(x^S)]\ge f(\mathbb{E}[x^S])$,
\begin{equation}\label{couldcel1}
\mathbb{E}[f(x^S)]-f^{\ast}=\mathbb{E}[f(x^S)]+\phi(\lambda^{\ast})\ \ge\ \langle\lambda^{\ast},r\rangle
\ \ge\ -\|\lambda^{\ast}\|_\infty\|r\|_1 ,
\end{equation}
the last step by H\"older's inequality. Third, for $S\ge2$,
\begin{equation}\label{couldcel2}
\sum_{s=0}^{S-1}\frac{1}{\tau_{1,s}}=\frac{S^2+7S}{4}\ \ge\ \frac{S^2+8S+16}{8}=\frac{\tau_2}{\tau_{1,S}^2},
\end{equation}
so $\Delta\ge c_\phi$, $\Delta\ge MS^2/4$, and $c_\phi=(M(S+4)^2)/8\ge MS^2/8$.

\medskip\noindent\textbf{Step 1. The objective error $\mathbb{E}[f(x^S)]-f(x^{\ast})$.}
Put $\lambda=\lambda^{\ast}$ in \eqref{master}. By \eqref{couldcel1},
$\langle r,\lambda^{\ast}\rangle\le\mathbb{E}[f(x^S)]-f^{\ast}$, so $\Delta\langle r,\lambda^{\ast}\rangle$ on
the right cancels the $f$-term on the left, leaving the dual-gap bound
\begin{equation}\label{dualgap}
c_\phi\big(\mathbb{E}[\phi(\widetilde{\lambda}^{S})]-\phi(\lambda^{\ast})\big)\ \le\ 2M B_0+9\overline{L}\|\lambda^{\ast}\|_2^2 .
\end{equation}
Putting instead $\lambda=0$ in \eqref{master} and discarding the nonnegative $\phi$-term,
\begin{equation}\label{primalup}
\Delta\big(\mathbb{E}[f(x^S)]-f^{\ast}\big)\ \le\ 2M B_0 .
\end{equation}
Dividing \eqref{dualgap} by $c_\phi$ and \eqref{primalup} by $\Delta$ and adding, with $\Delta\ge c_\phi$
(so $1/\Delta\le 1/(c_\phi)$) and $c_\phi\ge MS^2/8$,
\[
\big(\mathbb{E}[\phi(\widetilde{\lambda}^{S})]-\phi(\lambda^{\ast})\big)+\big(\mathbb{E}[f(x^S)]-f^{\ast}\big)
\ \le\ \frac{4M B_0+9\overline{L}\|\lambda^{\ast}\|_2^2}{c_\phi}
\ \le\ \frac{32 B_0}{S^2}+\frac{72\,\overline{L}\|\lambda^{\ast}\|_2^2}{MS^2}.
\]
Because $\mathbb{E}[\phi(\widetilde{\lambda}^{S})]\ge\phi(\lambda^{\ast})=-f^{\ast}$, the left side is at
least $\mathbb{E}[f(x^S)]-f^{\ast}$; hence
\[
\mathbb{E}[f(x^S)]-f(x^{\ast})=\mathcal{O}\!\left(\frac{\phi(0)-\phi(\lambda^{\ast})}{S^2}
+\frac{\overline{L}\|\lambda^{\ast}\|_2^2}{MS^2}\right).
\]

\medskip\noindent\textbf{Step 2. The constraint violation $\|\mathbb{E}[Ax^S-b]\|_1$.}
Fix any radius $R>0$ and take the \emph{norm-controlled} adversarial direction
$\lambda=-2R\,\mathrm{sign}(r)$, for which $\langle r,\lambda\rangle=-2R\|r\|_1$ and
$\|\lambda\|_2^2=4R^2\,|\mathrm{supp}(r)|\le 4dR^2$. The scaling keeps the test vector within the
$\ell_\infty$-radius $2R$, which controls the term $9\overline{L}\|\lambda\|_2^2$.
Using \eqref{couldcel1} and the nonnegativity of the dual-gap term, the left side of \eqref{master}
is at least $-\Delta\|\lambda^{\ast}\|_\infty\|r\|_1$. Hence
\[
-\Delta\|\lambda^{\ast}\|_\infty\|r\|_1\ \le\ -2\Delta R\|r\|_1+2M B_0+36\,\overline{L}\,dR^2 .
\]
Choosing $R\ge\|\lambda^\ast\|_\infty$ gives
$\Delta R\|r\|_1\le 2M B_0+36\,\overline{L}\,dR^2$. Dividing by $\Delta R$ and using
$\Delta\ge MS^2/4$,
\[
\big\|\mathbb{E}[Ax^S-b]\big\|_1\ \le\ \frac{8 B_0}{S^2R}+\frac{144\,\overline{L}\,dR}{MS^2}
=\mathcal{O}\!\left(\frac{\phi(0)-\phi(\lambda^{\ast})}{S^2R}+\frac{\overline{L}\,dR}{MS^2}\right).
\]
This proves the theorem for any positive radius $R$ with $R\ge\|\lambda^\ast\|_\infty$.
\end{proof}

\subsection{Proof of Corollary~\ref{theorem1cor}}
\begin{proof}
By $L^{\prime}$-smoothness of $\phi$ with respect to $\|\cdot\|_\infty$ and
$\nabla\phi(\lambda^\ast)=0$,
\[
\phi(0)-\phi(\lambda^{\ast})
\le
\langle \nabla \phi(\lambda^{\ast}),-\lambda^{\ast}\rangle
+\frac{L^{\prime}}{2}\|\lambda^{\ast}\|^2_\infty
=\frac{L^{\prime}}{2}\|\lambda^{\ast}\|^2_\infty.
\]

Applying Theorem~\ref{generaltheorem} gives
\[
\mathbb{E}[f(x^S)] -f(x^{\ast})
=\mathcal{O}\left(
\frac{L^\prime \|\lambda^\ast\|_\infty^2}{S^2}
+ \frac{\overline{L}\|\lambda^{\ast}\|_{2}^2}{MS^2}
\right),
\]
and, for any $R\ge\|\lambda^\ast\|_\infty$,
\[
\left\Vert \mathbb{E}[ Ax^{S}-b]\right\Vert_1
=\mathcal{O}\left(
\frac{L^\prime \|\lambda^\ast\|_\infty^2}{S^2R}
+ \frac{\overline{L}\,dR}{MS^2}
\right).
\]
\end{proof}
\subsection{Proof of Corollary~\ref{coro}}\label{proof:coro}
\begin{proof}
For the PDASGD output $x^{S}$, Corollary~\ref{theorem1cor} with $M=n$ gives
\[
\mathbb{E}[f(x^S)]-f(x^{\ast})
=\mathcal{O}\left(\frac{L^{\prime}\left\Vert \lambda^{\ast}\right\Vert_{\infty}^2}{S^2}
+ \frac{\overline{L}\left\Vert \lambda^{\ast}\right\Vert_{2}^2}{nS^2}\right)
=\mathcal{O}\left(\frac{(L^{\prime}+\overline{L})\left\Vert \lambda^{\ast}\right\Vert_{\infty}^2}{S^2}\right).
\]
Since $L^{\prime}$ and $\overline{L}$ have the same order,
\[
\mathbb{E}[f(x^S)] -f(x^{\ast})
=\mathcal{O}\left(\frac{L^{\prime}\left\Vert \lambda^{\ast}\right\Vert_{\infty}^2}{S^2}\right).
\]
Since one outer loop uses $n$ component steps, the total number of PDASGD component steps is $k=nS$.
The condition $\mathbb{E}[f(x^S)] -f(x^{\ast})\leq\epsilon$ is achieved with
\[
k=\mathcal{O}\left(n\sqrt{\frac{L^\prime}{\epsilon}}\left\Vert\lambda^\ast\right\Vert_\infty\right),
\]
component steps.

For the APDAGD output $x^{k}$,
\[
f(x^{k})-f(x^{\ast})
=\mathcal{O}\left(\frac{\widetilde{L}\left\Vert \lambda^{\ast}\right\Vert_2^2}{k^2}\right)
=\mathcal{O}\left(\frac{n\widetilde{L}\left\Vert \lambda^{\ast}\right\Vert_\infty^2}{k^2}\right).
\]
Since $L^{\prime}$ and $\widetilde{L}$ have the same order,
\[
f(x^{k})-f(x^\ast)=\mathcal{O}\left( \frac{nL^\prime\left\Vert \lambda^{\ast}\right\Vert_\infty^2}{k^2}\right).
\]
The deterministic method therefore reaches $f(x^k)-f(x^{\ast})\leq\epsilon$ after
\[
k=\mathcal{O}\left(\sqrt{n}\sqrt{\frac{L^\prime}{\epsilon}}\left\Vert\lambda^\ast\right\Vert_\infty\right),
\]
full-gradient iterations.

If one component step costs $\mathcal{O}(\mathcal{K})$ and one deterministic full-gradient step costs
$\mathcal{O}(n\mathcal{K})$, the total arithmetic costs are
\[
\mathcal{O}\left(n\mathcal{K}\sqrt{\frac{L^{\prime}}{\epsilon}}\|\lambda^\ast\|_\infty\right)
\quad\text{for PDASGD},
\]
and
\[
\mathcal{O}\left(n^{3/2}\mathcal{K}\sqrt{\frac{L^{\prime}}{\epsilon}}\|\lambda^\ast\|_\infty\right)
\quad\text{for APDAGD}.
\]
\end{proof}
\subsection{Proof of Proposition~\ref{gradientequal}}
\begin{proof}
\label{gradientequalproof}
Fix $v$ and define the row softmax probabilities
\[
p_{ij}(v)=
\frac{\exp\!\left(\frac{v_j-c_{ij}-\eta}{\eta}\right)}
{\sum_{\ell=1}^{n}\exp\!\left(\frac{v_\ell-c_{i\ell}-\eta}{\eta}\right)} .
\]
Equation~\eqref{closedformofprimalandsemidual} gives
\[
x(v)_{i+n(j-1)}=\alpha_i p_{ij}(v).
\]
Since $\sum_j p_{ij}(v)=1$, the first $n$ constraints are satisfied exactly:
\[
\big[Ax(v)-b\big]_{1:n}=0.
\]
The column-marginal block is
\[
\big[Ax(v)-b\big]_{n+j}
=\sum_{i=1}^n \alpha_i p_{ij}(v)-\beta_j,\qquad j=1,\ldots,n.
\]
Differentiating $G$ gives the same expression,
\[
[\nabla G(v)]_j
=-\beta_j+\sum_{i=1}^{n}\alpha_i p_{ij}(v).
\]
Hence $\nabla G(v)=\big[Ax(v)-b\big]_{n+1:2n}$, while
$\big[Ax(v)-b\big]_{1:n}=0$.
\end{proof}
\subsection{Proof of Proposition~\ref{proposition1}}
\begin{proof}
Fix $i$ and set
\[
p_j(v)=
\frac{\exp\!\left(\frac{v_j-c_{ij}-\eta}{\eta}\right)}
{\sum_{\ell=1}^{n}\exp\!\left(\frac{v_\ell-c_{i\ell}-\eta}{\eta}\right)} ,
\qquad p(v)\in\Delta_n .
\]
Then
\[
\nabla g_i(v)=n\alpha_i\big(p(v)-\beta\big),
\qquad
\nabla^2g_i(v)=\frac{n\alpha_i}{\eta}\big(\diag(p(v))-p(v)p(v)^\top\big).
\]
For any $y\in\mathbb R^n$,
\[
y^\top\nabla^2g_i(v)y
=\frac{n\alpha_i}{\eta}\left(\sum_{j=1}^n p_j(v)y_j^2
-\left(\sum_{j=1}^n p_j(v)y_j\right)^2\right)
=\frac{n\alpha_i}{\eta}\operatorname{Var}_{p(v)}(y).
\]
This quantity is nonnegative, so $g_i$ is convex. It also satisfies
\[
y^\top\nabla^2g_i(v)y
\le
\frac{n\alpha_i}{\eta}\sum_{j=1}^n p_j(v)y_j^2
\le
\frac{n\alpha_i}{\eta}\|y\|_2^2.
\]
Therefore $\|\nabla^2g_i(v)\|_2\le n\alpha_i/\eta$, and $g_i$ is
$L_i$-smooth with $L_i=n\alpha_i/\eta$. The average component smoothness is
\[
\overline{L}
=\frac{1}{n}\sum_{i=1}^{n}L_i
=\frac{1}{n}\sum_{i=1}^{n}\frac{n\alpha_i}{\eta}
=\frac{1}{\eta}.
\]
Lemma 1 in \cite{luo2023improved} shows that $G(v)$ is $5/\eta$-smooth with respect to
$\|\cdot\|_\infty$.
\end{proof}
\subsection{Proof of Corollary~\ref{coro2}}
\begin{proof}
Let $R^\prime$ denote an upper bound on the $\ell_\infty$ norm of an optimal dual variable for the
entropic OT problem. Lemma 3.2 in \cite{lin2019efficient} gives
\[
R^{\prime}\leq \eta\left(\frac{\left\Vert C\right\Vert_{\infty}}{\eta}+\ln n
-2\ln\min\left\{\min_i\alpha_i,\min_j\beta_j\right\}+\frac{1}{2}\right).
\]
Under the polynomial mass lower bounds in
Corollary~\ref{coro2}, $R'=\widetilde{\mathcal O}_\eta(1+\|C\|_\infty)$.

Since the semi-dual of entropic OT is $5/\eta$-smooth w.r.t. $\left\Vert \cdot\right\Vert_\infty$,
Corollary~\ref{coro} gives
\[
\mathcal{O}\left(n^{2.5}\sqrt{\frac{5}{\eta \epsilon}} R^\prime\right)
=\widetilde{\mathcal{O}}_\eta\left(\frac{n^{2.5}(1+\|C\|_\infty)}{\sqrt{\epsilon}}\right)
\]
for APDAGD, and
\[
\mathcal{O}\left(n^{2}\sqrt{\frac{5}{\eta \epsilon}} R^\prime\right)
=\widetilde{\mathcal{O}}_\eta\left(\frac{n^{2}(1+\|C\|_\infty)}{\sqrt{\epsilon}}\right)
\]
for PDASGD.
\end{proof}
\subsection{Proof of Proposition~\ref{rem1}}\label{app:proof-rem1}
\begin{proof}
Let
\[
\phi(\lambda)=\max_{x\in Q}\{-f(x)+\langle Ax,\lambda\rangle\}-\langle b,\lambda\rangle .
\]
Since $f$ is $\sigma$-strongly convex on the convex set $Q$, the maximizer
$x(\lambda)$ is unique for every $\lambda$ for which the maximum is finite. Danskin's theorem therefore
applies and gives
\[
\nabla\phi(\lambda)=Ax(\lambda)-b.
\]
The function $\phi$ is also the pointwise supremum of affine functions of $\lambda$ and is therefore convex.
This proves part (i).

For part (ii), let $x(\lambda)$ and $x(\lambda')$ be the two primal responses. The first-order
optimality conditions for the two strongly concave maximization problems imply, for all feasible
directions in $Q$,
\[
\langle \nabla f(x(\lambda))-A^\top\lambda,x(\lambda')-x(\lambda)\rangle\ge0,\qquad
\langle \nabla f(x(\lambda'))-A^\top\lambda',x(\lambda)-x(\lambda')\rangle\ge0.
\]
Adding the two inequalities yields
\[
\langle \nabla f(x(\lambda))-\nabla f(x(\lambda')),x(\lambda)-x(\lambda')\rangle
\le
\langle A^\top(\lambda-\lambda'),x(\lambda)-x(\lambda')\rangle .
\]
Strong convexity of $f$ gives
\[
\sigma\|x(\lambda)-x(\lambda')\|_E^2
\le
\langle \lambda-\lambda',A(x(\lambda)-x(\lambda'))\rangle
\le
\|\lambda-\lambda'\|_2\,\|A(x(\lambda)-x(\lambda'))\|_2 .
\]
By definition of the operator norm $\|A\|_{E,2}$,
\[
\|A(x(\lambda)-x(\lambda'))\|_2\le \|A\|_{E,2}\|x(\lambda)-x(\lambda')\|_E,
\]
and hence
\[
\|x(\lambda)-x(\lambda')\|_E\le \frac{\|A\|_{E,2}}{\sigma}\|\lambda-\lambda'\|_2 .
\]
Using the gradient identity from part (i),
\[
\|\nabla\phi(\lambda)-\nabla\phi(\lambda')\|_2
=\|A(x(\lambda)-x(\lambda'))\|_2
\le \frac{\|A\|_{E,2}^2}{\sigma}\|\lambda-\lambda'\|_2 .
\]
Thus $\phi$ is $L$-smooth with $L\le \|A\|_{E,2}^2/\sigma$.

For part (iii), assume $\phi=h^{-1}\sum_{i=1}^h\phi_i$ and each $\phi_i$ is convex and differentiable. The smoothness of
$\phi$ gives, for all $x,y$,
\[
\phi(y)-\phi(x)-\langle\nabla\phi(x),y-x\rangle\le \frac{L}{2}\|y-x\|_2^2.
\]
Multiplying by $h$ and using the finite-sum representation,
\[
\sum_{i=1}^h
\Big(\phi_i(y)-\phi_i(x)-\langle\nabla\phi_i(x),y-x\rangle\Big)
\le \frac{hL}{2}\|y-x\|_2^2 .
\]
Convexity of each $\phi_i$ makes every summand nonnegative. Therefore, for every fixed $i$,
\[
\phi_i(y)-\phi_i(x)-\langle\nabla\phi_i(x),y-x\rangle
\le \frac{hL}{2}\|y-x\|_2^2,
\]
which is exactly $hL$-smoothness of $\phi_i$ with respect to $\|\cdot\|_2$. Since
$L\le\|A\|_{E,2}^2/\sigma$, the average component smoothness satisfies
$\overline L=h^{-1}\sum_i hL\le h\|A\|_{E,2}^2/\sigma$.
\end{proof}
\section{Proof of the Main Auxiliary Estimate}\label{proofoffirstlemma}

For any integer $k \ge 0$, let
$\mathcal{F}_k = \sigma(y_0,z_0,\lambda_0,\ldots,y_k,z_k,\lambda_k)$ be the natural filtration generated
by the first $k$ iterates of PDASGD, and write
$\mathbb{E}[\cdot \mid k] := \mathbb{E}[\cdot \mid \mathcal{F}_k]$. Conditional inequalities below are
understood to hold almost surely with respect to $\mathcal F_k$.

Recall that PDASGD uses $\tau_{1,s} = 2/(s+4)$, $\tau_2 = 1/2$, and $\gamma_s = 1/(9\tau_{1,s}\overline{L})$. The Euclidean prox step for $z$ uses step length $\alpha_s=\gamma_s/2$. Hence
\[
\tau_{1,s}\ =\ \frac{1}{9\gamma_s\overline{L}}\ =\ \frac{1}{18\alpha_s\overline{L}}
\ \le\ \frac{1}{9\alpha_s\overline{L}},
\]
which is the step-size condition required by Lemma~E.4 of \cite{allen2017katyusha}. In that result,
the prox step length is denoted by $\alpha$; in PDASGD it is $\alpha_s=\gamma_s/2$. Lemma~E.2 of
\cite{allen2017katyusha} supplies the variance control for the same non-uniform sampling
$p_i = L_i/(h\overline{L})$.

\subsection{One-Step Estimates}
We use two one-step estimates before summing over the inner and outer loops.

Fix an iteration $k$. Conditional on $\mathcal F_k$, the variables $y_k$, $z_k$, and $\lambda_{k+1}$
are fixed, and the only randomness is the sampled component index $i$. The iteration is
   \begin{equation}
   \label{generalstep1}
    \lambda_{k+1}=\tau_{1}z_k+\tau_2\widetilde{\lambda}+(1-\tau_{1}-\tau_2)y_k,
   \end{equation}
   \begin{equation*}
   \widetilde{\nabla}_{k+1}=u+\frac{\nabla \phi_{i}(\lambda_{k+1})-\nabla \phi_{i}(\widetilde{\lambda})}{hp_i},
   \end{equation*}
   \begin{equation*}
    z_{k+1}= z_{k}-\frac{\gamma\widetilde{\nabla}_{k+1}}{2},
    \end{equation*}
    \begin{equation*}
    y_{k+1}=\lambda_{k+1}-\frac{\widetilde{\nabla}_{k+1}}{9\overline{L}}.   
   \end{equation*}

\begin{lemma}
  \label{generallemma1}
\begin{equation*}
   \begin{aligned}
   &\gamma\langle\nabla \phi(\lambda_{k+1}),z_k-\lambda\rangle\\
   &\leq \frac{\gamma}{\tau_1} \left( \phi(\lambda_{k+1})-\mathbb{E}\left[\phi(y_{k+1})\big\vert k\right] +\tau_2 \phi(\widetilde{\lambda})-\tau_2 \phi(\lambda_{k+1})-\tau_2\langle \nabla \phi(\lambda_{k+1}),\widetilde{\lambda}-\lambda_{k+1}\rangle \right)
   \\&+\left\Vert z_k-\lambda\right\Vert^2_2-\mathbb{E}\left[\left\Vert z_{k+1}-\lambda\right\Vert^2_2\big\vert k\right],
   \end{aligned}
   \end{equation*}
holds for all $\lambda\in\mathbb{R}^d$.
   \end{lemma}
\begin{proof}
We apply Lemma~E.4 of \cite{allen2017katyusha} with $\psi \equiv 0$, smooth objective
$f=\phi$, and Euclidean Bregman divergence $V_{z_k}(u)=\|z_k-u\|_2^2/2$. The $z$-update in that
lemma is $z_{k+1}=z_k-\alpha\widetilde\nabla_{k+1}$, which coincides with Step~10 of PDASGD when
$\alpha=\gamma/2$. The dual norm is therefore $\|\cdot\|_2$. The step-size condition
$\tau_1\le 1/(9\alpha\overline L)$ holds because $\gamma=1/(9\tau_1\overline L)$ implies
$\tau_1=1/(18\alpha\overline L)$. Lemma~E.2 of \cite{allen2017katyusha} gives the second-moment
control for the SVRG estimator under the non-uniform sampling $p_i=L_i/(h\overline L)$. Substituting
these choices into Lemma~E.4 yields the stated inequality.
\end{proof}
  \begin{lemma}
  \label{generallemma2}
  \begin{equation*}
  \begin{aligned}
  \mathbb{E}\left[\phi(y_{k+1})\big\vert k\right]&\leq \tau_1 \left(\phi(\lambda_{k+1})+\langle \nabla \phi(\lambda_{k+1}),\lambda-\lambda_{k+1}\rangle\right)+\tau_2 \phi(\widetilde{\lambda})+(1-\tau_1-\tau_2)\phi(y_k)\\
  &+\frac{\tau_1}{\gamma}\left\Vert{z_k}- \lambda\right\Vert^2_2-\frac{\tau_1}{\gamma}\mathbb{E}\left[\left\Vert z_{k+1}-\lambda \right\Vert^2_2\big\vert k\right]
  \end{aligned}
  \end{equation*}
    holds for all $\lambda\in\mathbb{R}^d$.
  \end{lemma}
 \begin{proof}
Following the argument in the proof of Lemma~2.7 of \cite{allen2017katyusha}, expand
$\gamma\langle \nabla\phi(\lambda_{k+1}),\lambda_{k+1}-\lambda\rangle$ about $z_k$ using
\eqref{generalstep1}. Convexity of $\phi$ bounds the term involving $y_k$:
   \begin{equation}\label{SM1}
       \begin{aligned}
       &\gamma\langle \nabla \phi(\lambda_{k+1}),\lambda_{k+1}-\lambda\rangle\\
       =&\gamma \langle \nabla \phi(\lambda_{k+1}),\lambda_{k+1}-z_{k}\rangle +\gamma\langle \nabla \phi(\lambda_{k+1}),z_k-\lambda\rangle\\
       \overset{(a)}{=}&\frac{\gamma\tau_2}{\tau_1}\langle\nabla \phi(\lambda_{k+1}),\widetilde{\lambda}-\lambda_{k+1}\rangle+\frac{\gamma(1-\tau_1-\tau_2)}{\tau_1}\langle\nabla \phi(\lambda_{k+1}),y_k-\lambda_{k+1}\rangle+\gamma \langle\nabla \phi(\lambda_{k+1}),z_k-\lambda\rangle\\
       \overset{(b)}{\leq}&\frac{\gamma\tau_2}{\tau_1}\langle\nabla \phi(\lambda_{k+1}),\widetilde{\lambda}-\lambda_{k+1}\rangle+\frac{\gamma(1-\tau_1-\tau_2)}{\tau_1}(\phi(y_k)-\phi(\lambda_{k+1}))+\gamma \langle\nabla \phi(\lambda_{k+1}),z_k-\lambda\rangle,
       \end{aligned}
   \end{equation}
   where (a) uses \eqref{generalstep1} and the combination 
   \[\lambda_{k+1}=\left(\frac{1-\tau_2}{\tau_1}-\frac{1-\tau_1-\tau_2}{\tau_1}\right)\lambda_{k+1},\] 
   and (b) uses the convexity of $\phi$.

   Applying Lemma \ref{generallemma1} to \eqref{SM1} gives
   \begin{equation*}
   \begin{aligned}
   &\gamma\langle \nabla \phi(\lambda_{k+1}),\lambda_{k+1}-\lambda\rangle\\
     &\leq\frac{\gamma\tau_2}{\tau_1}\langle\nabla \phi(\lambda_{k+1}),\widetilde{\lambda}-\lambda_{k+1}\rangle+\frac{\gamma(1-\tau_1-\tau_2)}{\tau_1}(\phi(y_k)-\phi(\lambda_{k+1}))\\
     &+\frac{\gamma}{\tau_1}\left(\phi(\lambda_{k+1})-\mathbb{E}\left[\phi(y_{k+1})\big\vert k\right]+\tau_2 \phi(\widetilde{\lambda})-\tau_2 \phi(\lambda_{k+1})-\tau_2\langle \nabla \phi(\lambda_{k+1}),\widetilde{\lambda}-\lambda_{k+1}\rangle\right)\\
    &+ \left\Vert z_k-\lambda\right\Vert^2_2-\mathbb{E}\left[\left\Vert z_{k+1}-\lambda\right\Vert^2_2\big\vert k\right]\\
    &=\frac{\gamma(1-\tau_1-\tau_2)}{\tau_1}\phi(y_k)+\gamma \phi(\lambda_{k+1})-\frac{\gamma}{\tau_1}\mathbb{E}\left[\phi(y_{k+1})\big\vert k\right]+\frac{\gamma\tau_2}{\tau_1}\phi(\widetilde{\lambda})\\
    &+\left\Vert z_k-\lambda\right\Vert^2_2- \mathbb{E}\left[\left\Vert z_{k+1}-\lambda\right\Vert^2_2\big\vert k\right].
   \end{aligned}
   \end{equation*}
   
   Rearranging gives
   \begin{equation*}
   \begin{aligned}
   \mathbb{E}\left[\phi(y_{k+1})\big\vert k\right]&\leq \tau_1 \left(\phi(\lambda_{k+1})+\langle \nabla \phi(\lambda_{k+1}),\lambda-\lambda_{k+1}\rangle\right)+\tau_2\phi(\widetilde{\lambda})+(1-\tau_1-\tau_2)\phi(y_k)\\
   &+\frac{\tau_1}{\gamma}\left\Vert z_k-\lambda\right\Vert^2_2-\frac{\tau_1}{\gamma}\mathbb{E}\left[\left\Vert z_{k+1}-\lambda\right\Vert^2_2\big\vert k\right].
   \end{aligned}
\end{equation*}
\end{proof}

\subsection{Telescoping Argument}
We now prove Lemma \ref{firstlemma} using Lemmas \ref{generallemma1} and \ref{generallemma2}.
\begin{proof}
We replace $\tau_1$ by $\tau_{1,s}$ and sum both sides of the inequality in Lemma~\ref{generallemma2}
for $k=sM,\cdots,sM+M-1$, obtaining
   \begin{equation}
   \begin{aligned}
   \label{firstsum}
   &\sum_{j=1}^{M}\mathbb{E}\left[\phi(y_{sM+j})\big\vert sM+j-1\right]\\
   &\leq \tau_{1,s} \sum_{j=1}^{M}\left(\phi(\lambda_{sM+j})+\langle \nabla \phi(\lambda_{sM+j}),\lambda-\lambda_{sM+j}\rangle\right)+\tau_{2}M\phi(\widetilde{\lambda})+(1-\tau_{1,s}-\tau_{2})\sum_{j=0}^{M-1}\phi(y_{sM+j})\\
   &+\frac{\tau_{1,s}}{\gamma_s}\sum_{j=0}^{M-1}\left\Vert z_{sM+j}-\lambda\right\Vert^2_2-\frac{\tau_{1,s}}{\gamma_s}\sum_{j=1}^{M} \mathbb{E}\left[\left\Vert z_{sM+j}-\lambda\right\Vert^2_2\big\vert sM+j-1\right],
   \end{aligned}
\end{equation}
where $\mathbb{E}\left[\cdot\big\vert sM+j-1\right]$ denotes the expectation conditional on previous $sM+j-1$ steps.

Step 13 in PDASGD implies $\widetilde{\lambda}^{s}=M^{-1}\sum_{j=1}^{M}y_{(s-1)M+j}$ for $s\ge1$.
For $s=0$ we use the initialization convention $y_p=0$ for all $p\le0$, so the same identity gives
$\widetilde{\lambda}^0=0$. Substituting this expression into \eqref{firstsum} and applying Jensen's
inequality to the convex function $\phi$ gives
  \begin{equation*}
   \begin{aligned}
   \sum_{j=1}^{M}\mathbb{E}\left[\phi(y_{sM+j})\big\vert sM+j-1\right]&\leq \tau_{1,s} \sum_{j=1}^{M}(\phi(\lambda_{sM+j})+\langle \nabla \phi(\lambda_{sM+j}),\lambda-\lambda_{sM+j}\rangle)\\
   &+\tau_{2}\sum_{j=1}^{M}\phi(y_{(s-1)M+j})+(1-\tau_{1,s}-\tau_{2})\sum_{j=0}^{M-1}\phi(y_{sM+j})\\
   &+\frac{\tau_{1,s}}{\gamma_s}\sum_{j=0}^{M-1}\left\Vert z_{sM+j}-\lambda\right\Vert^2_2-\frac{\tau_{1,s}}{\gamma_s}\sum_{j=1}^{M}\mathbb{E}\left[\left\Vert z_{sM+j}-\lambda\right\Vert^2_2\big\vert sM+j-1\right].
   \end{aligned}
\end{equation*}

   We now take the total expectation of both sides, without conditioning. For the left-hand side, the tower rule gives $\mathbb{E}\big[\mathbb{E}[\,\cdot\mid sM+j-1]\big] = \mathbb{E}[\,\cdot\,]$, which removes all conditional expectations $\mathbb{E}[\,\cdot\mid sM+j-1]$. The $z$-terms then telescope, collapsing the inner sum $\sum_{j=0}^{M-1}\|z_{sM+j}-\lambda\|^2 - \sum_{j=1}^{M}\mathbb{E}[\|z_{sM+j}-\lambda\|^2]$ to $\|z_{sM}-\lambda\|^2 - \mathbb{E}[\|z_{sM+M}-\lambda\|^2]$. We obtain
   \begin{equation*}
   \begin{aligned}
   \mathbb{E}\left[\sum_{j=1}^{M}\phi(y_{sM+j})\right]&\leq \tau_{1,s} \sum_{j=1}^{M}\mathbb{E}\left[(\phi(\lambda_{sM+j})+\langle \nabla \phi(\lambda_{sM+j}),\lambda-\lambda_{sM+j}\rangle)\right]\\
   &+\tau_{2}\mathbb{E}\left[\sum_{j=1}^{M}\phi(y_{(s-1)M+j})\right]+(1-\tau_{1,s}-\tau_{2})\mathbb{E}\left[\sum_{j=0}^{M-1}\phi(y_{sM+j})\right]\\
   &+\frac{\tau_{1,s}}{\gamma_s}\mathbb{E}\left[\left\Vert z_{sM}-\lambda\right\Vert^2_2\right]-\frac{\tau_{1,s}}{\gamma_s}\mathbb{E}\left[\left\Vert z_{sM+M}-\lambda\right\Vert^2_2\right].
   \end{aligned}
\end{equation*}

Divide both sides by $\tau_{1,s}^2$, rewrite $\sum_{j=0}^{M-1}\phi(y_{sM+j})$  as $\sum_{j=1}^{M}\phi(y_{sM+j})+\phi(y_{sM})-\phi(y_{sM+M})$ on the right side, and rearrange the terms.
 \begin{equation*}
   \begin{aligned}
   \label{rewritefirstsum}
   &\mathbb{E}\left[\frac{\tau_{1,s}+\tau_{2}}{\tau_{1,s}^2}\sum_{j=1}^{M}\phi(y_{sM+j})+\frac{1-\tau_{1,s}-\tau_{2}}{\tau_{1,s}^2}\phi(y_{sM+M})\right]\\
   &\leq \frac{1}{\tau_{1,s}} \sum_{j=1}^{M}\mathbb{E}\left[(\phi(\lambda_{sM+j})+\langle \nabla \phi(\lambda_{sM+j}),\lambda-\lambda_{sM+j}\rangle)\right]+\frac{\tau_{2}}{\tau_{1,s}^2}\mathbb{E}\left[\sum_{j=1}^{M}\phi(y_{(s-1)M+j})\right]\\
   &+\frac{1-\tau_{1,s}-\tau_{2}}{\tau_{1,s}^2}\mathbb{E}\left[\phi(y_{sM})\right]
   +9\overline{L}\mathbb{E}\left[\left\Vert z_{sM}-\lambda\right\Vert^2_2\right]-9\overline{L}\mathbb{E}\left[\left\Vert z_{sM+M}-\lambda\right\Vert^2_2\right].
   \end{aligned}
\end{equation*}

Split the first term on the left as
$\sum_{j=1}^{M}\phi(y_{sM+j})=\sum_{j=1}^{M-1}\phi(y_{sM+j})+\phi(y_{sM+M})$ and split the
second term on the right as
$\sum_{j=1}^{M}\phi(y_{(s-1)M+j})=\sum_{j=1}^{M-1}\phi(y_{(s-1)M+j})+\phi(y_{sM})$.
\begin{equation}
   \begin{aligned}\label{ineqdistribute}
   &\mathbb{E}\left[\frac{\tau_{1,s}+\tau_2}{\tau_{1,s}^2}\sum_{j=1}^{M-1}\phi(y_{sM+j})\right]+\mathbb{E}\left[\frac{1}{\tau_{1,s}^2}\phi(y_{sM+M})\right]\\
   &\leq \frac{1}{\tau_{1,s}} \sum_{j=1}^{M}\mathbb{E}\left[(\phi(\lambda_{sM+j})+\langle \nabla \phi(\lambda_{sM+j}),\lambda-\lambda_{sM+j}\rangle)\right]+\frac{\tau_{2}}{\tau_{1,s}^2}\mathbb{E}\left[\sum_{j=1}^{M-1}\phi(y_{(s-1)M+j})\right]\\
   &+\frac{1-\tau_{1,s}}{\tau_{1,s}^2}\mathbb{E}\left[\phi(y_{sM})\right]
   +9\overline{L}\mathbb{E}\left[\left\Vert z_{sM}-\lambda\right\Vert^2_2\right]-9\overline{L}\mathbb{E}\left[\left\Vert z_{sM+M}-\lambda\right\Vert^2_2\right].
   \end{aligned}
\end{equation}

We distribute the optimal value term $\phi(\lambda^{\ast})$ over $\left( (M-1)(\tau_{1,s}+\tau_{2})/\tau_{1,s}^2+1/\tau_{1,s}^2\right)$-times on each side of inequality \eqref{ineqdistribute}.
\begin{equation}\label{stillhold}
    \begin{aligned}
    &\mathbb{E}\left[\frac{\tau_{1,s}+\tau_2}{\tau_{1,s}^2}\sum_{j=1}^{M-1}\left(\phi(y_{sM+j})-\phi(\lambda^{\ast})\right)\right]+\mathbb{E}\left[\frac{1}{\tau_{1,s}^2}\left(\phi(y_{sM+M})-\phi(\lambda^{\ast})\right)\right]\\
   &\leq \frac{1}{\tau_{1,s}} \sum_{j=1}^{M}\mathbb{E}\left[\left(\phi(\lambda_{sM+j})-\phi(\lambda^{\ast})+\langle \nabla \phi(\lambda_{sM+j}),\lambda-\lambda_{sM+j}\rangle\right)\right]\\
   &+\frac{\tau_{2}}{\tau_{1,s}^2}\mathbb{E}\left[\sum_{j=1}^{M-1}\left(\phi(y_{(s-1)M+j})-\phi(\lambda^{\ast})\right)\right]\\
   &+\frac{1-\tau_{1,s}}{\tau_{1,s}^2}\mathbb{E}\left[\left(\phi(y_{sM})-\phi(\lambda^{\ast})\right)\right]
   +9\overline{L}\mathbb{E}\left[\left\Vert z_{sM}-\lambda\right\Vert^2_2\right]-9\overline{L}\mathbb{E}\left[\left\Vert z_{sM+M}-\lambda\right\Vert^2_2\right].
    \end{aligned}
\end{equation}

From $\tau_{1,s}=2/(s+4)$ and $\tau_2=1/2$, we have
\begin{equation}\label{inqus}
\frac{1}{\tau_{1,s}^2}\geq \frac{1-\tau_{1,s+1}}{\tau_{1,s+1}^2},
\qquad
\frac{\tau_{1,s}+\tau_2}{\tau_{1,s}^2}\geq\frac{\tau_2}{\tau_{1,s+1}^2}.
\end{equation}

Substituting the inequalities \eqref{inqus} into \eqref{stillhold} gives
\begin{equation}\label{rea}
    \begin{aligned}
    &\mathbb{E}\left[\frac{\tau_2}{\tau_{1,s+1}^2}\sum_{j=1}^{M-1}\left(\phi(y_{sM+j})-\phi(\lambda^{\ast})\right)\right]+\mathbb{E}\left[\frac{1-\tau_{1,s+1}}{\tau_{1,s+1}^2}\left(\phi(y_{sM+M})-\phi(\lambda^{\ast})\right)\right]\\
   &\leq \frac{1}{\tau_{1,s}} \sum_{j=1}^{M}\mathbb{E}\left[\left(\phi(\lambda_{sM+j})-\phi(\lambda^{\ast})+\langle \nabla \phi(\lambda_{sM+j}),\lambda-\lambda_{sM+j}\rangle\right)\right]\\
   &+\frac{\tau_{2}}{\tau_{1,s}^2}\mathbb{E}\left[\sum_{j=1}^{M-1}\left(\phi(y_{(s-1)M+j})-\phi(\lambda^{\ast})\right)\right]\\
   &+\frac{1-\tau_{1,s}}{\tau_{1,s}^2}\mathbb{E}\left[\left(\phi(y_{sM})-\phi(\lambda^{\ast})\right)\right]
   +9\overline{L}\mathbb{E}\left[\left\Vert z_{sM}-\lambda\right\Vert^2_2\right]-9\overline{L}\mathbb{E}\left[\left\Vert z_{sM+M}-\lambda\right\Vert^2_2\right].
    \end{aligned}
\end{equation}

Rearranging the terms of \eqref{rea} gives
\begin{equation}\label{rea2}
    \begin{aligned}
    &\frac{\tau_2}{\tau_{1,s+1}^2}\mathbb{E}\left[\sum_{j=1}^{M-1}\left(\phi(y_{sM+j})-\phi(\lambda^{\ast})\right)\right] - \frac{\tau_{2}}{\tau_{1,s}^2}\mathbb{E}\left[\sum_{j=1}^{M-1}\left(\phi(y_{(s-1)M+j})-\phi(\lambda^{\ast})\right)\right]\\
    &+\frac{1-\tau_{1,s+1}}{\tau_{1,s+1}^2}\mathbb{E}\left[\left(\phi(y_{sM+M})-\phi(\lambda^{\ast})\right)\right]-\frac{1-\tau_{1,s}}{\tau_{1,s}^2}\mathbb{E}\left[\left(\phi(y_{sM})-\phi(\lambda^{\ast})\right)\right]\\
   &\leq \frac{1}{\tau_{1,s}} \sum_{j=1}^{M}\mathbb{E}\left[\left(\phi(\lambda_{sM+j})-\phi(\lambda^{\ast})+\langle \nabla \phi(\lambda_{sM+j}),\lambda-\lambda_{sM+j}\rangle\right)\right]\\
   &   +9\overline{L}\mathbb{E}\left[\left\Vert z_{sM}-\lambda\right\Vert^2_2\right]-9\overline{L}\mathbb{E}\left[\left\Vert z_{sM+M}-\lambda\right\Vert^2_2\right].
    \end{aligned}
\end{equation}
 For $s=0$, the initialization convention $y_p=0$ for $p\le0$ gives
\begin{equation}\label{zero}
\sum_{j=1}^{M-1}\left(\mathbb{E}\left[\phi(y_{(s-1)M+j})\right]-\phi(\lambda^{\ast})\right)
= M(\phi(\widetilde{\lambda}^0)-\phi(\lambda^\ast))-(\phi(y_0)-\phi(\lambda^\ast)),
\end{equation}
which equals $(M-1)(\phi(0)-\phi(\lambda^\ast))$ because $y_0=\widetilde\lambda^0=0$.

Summing \eqref{rea2} over $s = 0,\cdots, S-1$ gives
\begin{equation}\label{telescope}
    \begin{aligned}
    &\sum_{s=0}^{S-1}\left(\frac{\tau_2}{\tau_{1,s+1}^2}\mathbb{E}\left[\sum_{j=1}^{M-1}\left(\phi(y_{sM+j})-\phi(\lambda^{\ast})\right)\right] - \frac{\tau_{2}}{\tau_{1,s}^2}\mathbb{E}\left[\sum_{j=1}^{M-1}\left(\phi(y_{(s-1)M+j})-\phi(\lambda^{\ast})\right)\right]\right)\\
    &+\sum_{s=0}^{S-1}\left(\frac{1-\tau_{1,s+1}}{\tau_{1,s+1}^2}\mathbb{E}\left[\left(\phi(y_{sM+M})-\phi(\lambda^{\ast})\right)\right]-\frac{1-\tau_{1,s}}{\tau_{1,s}^2}\mathbb{E}\left[\left(\phi(y_{sM})-\phi(\lambda^{\ast})\right)\right]\right)\\
   &\leq \sum_{s=0}^{S-1}\left(\frac{1}{\tau_{1,s}} \sum_{j=1}^{M}\mathbb{E}\left[\left(\phi(\lambda_{sM+j})-\phi(\lambda^{\ast})+\langle \nabla \phi(\lambda_{sM+j}),\lambda-\lambda_{sM+j}\rangle\right)\right]\right)\\
   &   +\sum_{s=0}^{S-1}\left(9\overline{L}\mathbb{E}\left[\left\Vert z_{sM}-\lambda\right\Vert^2_2\right]-9\overline{L}\mathbb{E}\left[\left\Vert z_{sM+M}-\lambda\right\Vert^2_2\right]\right).
    \end{aligned}
\end{equation}

The intermediate terms in \eqref{telescope} cancel. Together with \eqref{zero}, this gives
\begin{equation}\label{final}
    \begin{aligned}
    &\frac{\tau_2}{\tau_{1,S}^2}\mathbb{E}\left[\sum_{j=1}^{M-1}\left(\phi(y_{(S-1)M+j})-\phi(\lambda^{\ast})\right)\right] - \frac{\tau_{2}}{\tau_{1,0}^2}\left( M(\phi(\widetilde{\lambda}^0)-\phi(\lambda^\ast))-(\phi(y_0)-\phi(\lambda^\ast))\right)\\
    &+\frac{1-\tau_{1,S}}{\tau_{1,S}^2}\mathbb{E}\left[\left(\phi(y_{SM})-\phi(\lambda^{\ast})\right)\right]-\frac{1-\tau_{1,0}}{\tau_{1,0}^2}\left(\left(\phi(y_{0})-\phi(\lambda^{\ast})\right)\right)\\
   &\leq \sum_{s=0}^{S-1}\left(\frac{1}{\tau_{1,s}} \sum_{j=1}^{M}\mathbb{E}\left[\left(\phi(\lambda_{sM+j})-\phi(\lambda^{\ast})+\langle \nabla \phi(\lambda_{sM+j}),\lambda-\lambda_{sM+j}\rangle\right)\right]\right)\\
   &   +9\overline{L}\mathbb{E}\left[\left\Vert z_{0}-\lambda\right\Vert^2_2\right]-9\overline{L}\mathbb{E}\left[\left\Vert z_{SM}-\lambda\right\Vert^2_2\right].
    \end{aligned}
\end{equation}

Rearranging the terms gives
 \begin{equation}
 \label{generaltelecoperesults}
   \begin{aligned}
   &\frac{\tau_2}{\tau_{1,S}^2}\mathbb{E}\left[\sum_{j=1}^{M-1}\left(\phi(y_{(S-1)M+j})-\phi(\lambda^{\ast})\right)\right]+\frac{1-\tau_{1,S}}{\tau_{1,S}^2}\mathbb{E}\left[\phi(y_{SM})-\phi(\lambda^{\ast})\right]
   \\
   &\leq \sum_{s=0}^{S-1}\frac{1}{\tau_{1,s}} \sum_{j=1}^{M}\mathbb{E}\left[(\phi(\lambda_{sM+j})-\phi(\lambda^{\ast})+\langle \nabla \phi(\lambda_{sM+j}),\lambda-\lambda_{sM+j}\rangle)\right]\\
   &+\frac{1-\tau_{1,0}-\tau_2}{\tau_{1,0}^2}\left(\phi(y_{0})-\phi(\lambda^{\ast})\right)+ \frac{\tau_{2}M}{\tau_{1,0}^2}(\phi(\widetilde{\lambda}^0)-\phi(\lambda^{\ast}))+9\overline{L}\left\Vert z_{0}-\lambda\right\Vert^2_2-9\overline{L}\mathbb{E}\left[\left\Vert z_{SM}-\lambda\right\Vert^2_2\right].
   \end{aligned}
\end{equation}

Next, express $\phi(\lambda)$ as
\begin{equation}
\label{phiexpression}
\begin{aligned}
\phi(\lambda)&=-f(x(\lambda))+\langle\lambda,Ax(\lambda)-b\rangle\overset{(a)}{=}-f(x(\lambda))+\langle \lambda,\nabla \phi(\lambda)\rangle,
\end{aligned}
\end{equation}
where (a) uses Assumption \ref{assume} ($\nabla\phi(\lambda)=Ax(\lambda)-b$). Since $\nabla\phi(\lambda^\ast)=0$,
\begin{equation}
\label{phiexpression2}
\phi(\lambda^\ast)=-f(x(\lambda^\ast))\end{equation}

Equation~\eqref{phiexpression} and Assumption~\ref{assume}
($\nabla\phi(\lambda)=Ax(\lambda)-b$) also give
\begin{equation}
\label{difffinal}
\begin{aligned}
    \phi(\lambda_{sM+j})+\langle \nabla \phi(\lambda_{sM+j}),\lambda-\lambda_{sM+j}\rangle
    =-f(x(\lambda_{sM+j}))+\langle Ax(\lambda_{sM+j})-b,\lambda\rangle.
    \end{aligned}
\end{equation}

Next, we introduce ancillary variables
\[
x^S_a=
\frac{\sum_{s=0}^{S-1}\frac{1}{\tau_{1,s}}\sum_{j=1}^{M}x(\lambda_{sM+j})}{M\sum_{s=0}^{S-1}\frac{1}{\tau_{1,s}}},
\qquad
f^S_a=
\frac{\sum_{s=0}^{S-1}\frac{1}{\tau_{1,s}}\sum_{j=1}^{M}f(x(\lambda_{sM+j}))}{M\sum_{s=0}^{S-1}\frac{1}{\tau_{1,s}}}.
\]

Substituting \eqref{phiexpression2} and \eqref{difffinal} into the first term on the right side of
\eqref{generaltelecoperesults} gives
\begin{equation}
    \begin{aligned}
    \label{righthand}
    &\sum_{s=0}^{S-1}\frac{1}{\tau_{1,s}} \sum_{j=1}^{M}\mathbb{E}\left[(\phi(\lambda_{sM+j})-\phi(\lambda^{\ast})+\langle \nabla \phi(\lambda_{sM+j}),\lambda-\lambda_{sM+j}\rangle)\right]\\
    &=\sum_{s=0}^{S-1}\frac{1}{\tau_{1,s}}\sum_{j=1}^{M}\left(\mathbb{E}\left[-f(x(\lambda_{sM+j}))+f(x(\lambda^{\ast}))+\langle Ax(\lambda_{sM+j})-b,\lambda\rangle\right]\right)\\
    &=\left(M\sum_{s=0}^{S-1}\frac{1}{\tau_{1,s}}\right) \left(-\left(\mathbb{E}[f_a^S]-f(x(\lambda^{\ast}))\right)+\langle A\mathbb{E}[x_a^{S}]-b,\lambda\rangle\right).
    \end{aligned}
\end{equation}

    By the choice of $\tau_{1,s}=2/(s+4)$ and $\tau_2=1/2$,
    \[\frac{\tau_2}{\tau_{1,S}^2}\leq\frac{1-\tau_{1,S}}{\tau_{1,S}^2}.\] 
    
    Hence the left side of \eqref{generaltelecoperesults} satisfies
    \begin{equation}
    \label{lefthand}
    \begin{aligned}
   &\frac{\tau_2}{\tau_{1,S}^2}\mathbb{E}\left[\sum_{j=1}^{M-1}\left(\phi(y_{(S-1)M+j})-\phi(\lambda^{\ast})\right)\right]+\frac{1-\tau_{1,S}}{\tau_{1,S}^2}\mathbb{E}\left[\phi(y_{SM})-\phi(\lambda^{\ast})\right]
   \\&\geq  \frac{\tau_2}{\tau_{1,S}^2}\mathbb{E}\left[\sum_{j=1}^{M}\left(\phi(y_{(S-1)M+j})-\phi(\lambda^{\ast})\right)\right].
       \end{aligned}
    \end{equation}
    
 Substituting \eqref{righthand} and \eqref{lefthand} into \eqref{generaltelecoperesults} gives
    \begin{equation*}
   \begin{aligned}
   &\frac{\tau_2}{\tau_{1,S}^2}\mathbb{E}\left[\sum_{j=1}^{M}\left(\phi(y_{(S-1)M+j})-\phi(\lambda^{\ast})\right)\right]\\
   &\leq \left(M\sum_{s=0}^{S-1}\frac{1}{\tau_{1,s}}\right) \left(-\left(\mathbb{E}[f_a^S]-f(x(\lambda^{\ast}))\right)\right)+\left(M\sum_{s=0}^{S-1}\frac{1}{\tau_{1,s}}\right)\langle A\mathbb{E}[x_a^{S}]-b,\lambda\rangle\\ &+\frac{\tau_{2}M}{\tau_{1,0}^2}\left(\phi(\widetilde{\lambda}^0)-\phi(\lambda^{\ast})\right)+\frac{1-\tau_{1,0}-\tau_2}{\tau_{1,0}^2}\left(\phi(y_{0})-\phi(\lambda^{\ast})\right)+9\overline{L}\left\Vert z_{0}-\lambda\right\Vert^2_2-9\overline{L}\mathbb{E}\left[\left\Vert z_{SM}-\lambda\right\Vert^2_2\right].
   \end{aligned}
\end{equation*}

By definition of $\widetilde{\lambda}^S=1/M\sum_{j=1}^M y_{(S-1)M+j}$ and convexity of $\phi$,
  \begin{equation}
   \begin{aligned}
   \label{defandconvex}
   &\frac{\tau_2M}{\tau_{1,S}^2}\mathbb{E}\left[\phi(\widetilde{\lambda}^{S})-\phi(\lambda^{\ast})\right]\leq \frac{\tau_2}{\tau_{1,S}^2}\mathbb{E}\left[\sum_{j=1}^{M}\left(\phi(y_{(S-1)M+j})-\phi(\lambda^{\ast})\right)\right]\\
   &\leq\left(M\sum_{s=0}^{S-1}\frac{1}{\tau_{1,s}}\right) \left(-\left(\mathbb{E}[f_a^S]-f(x(\lambda^{\ast}))\right)\right)+\left(M\sum_{s=0}^{S-1}\frac{1}{\tau_{1,s}}\right)\langle A\mathbb{E}[x_a^{S}]-b,\lambda\rangle\\ &+\frac{\tau_{2}M}{\tau_{1,0}^2}\left(\phi(\widetilde{\lambda}^0)-\phi(\lambda^{\ast})\right)+\frac{1-\tau_{1,0}-\tau_2}{\tau_{1,0}^2}\left(\phi(y_{0})-\phi(\lambda^{\ast})\right)+9\overline{L}\left\Vert z_{0}-\lambda\right\Vert^2_2-9\overline{L}\mathbb{E}\left[\left\Vert z_{SM}-\lambda\right\Vert^2_2\right].
   \end{aligned}
\end{equation}

Notice that
\begin{equation}\label{xs}
x^{S}=
\frac{\sum_{s=0}^{S-1}\frac{1}{\tau_{1,s}}x(\widehat{\lambda}_s)}{\sum_{s=0}^{S-1}\frac{1}{\tau_{1,s}}},
\end{equation}
\[
x^S_a=
\frac{\sum_{s=0}^{S-1}\frac{1}{\tau_{1,s}}\sum_{j=1}^{M}x(\lambda_{sM+j})}{M\sum_{s=0}^{S-1}\frac{1}{\tau_{1,s}}},
\qquad
f^S_a=
\frac{\sum_{s=0}^{S-1}\frac{1}{\tau_{1,s}}\sum_{j=1}^{M}f(x(\lambda_{sM+j}))}{M\sum_{s=0}^{S-1}\frac{1}{\tau_{1,s}}},
\]
and $\widehat{\lambda}_s$ is chosen uniformly (with probability $1/M$) from
$\{\lambda_{sM+1},\ldots,\lambda_{sM+M}\}$, conditionally on the inner iterates generated during
outer loop $s$. Therefore,
\begin{equation}\label{uselater2}  \mathbb{E}[ x(\widehat{\lambda}_s)]=\frac{1}{M}\sum_{j=1}^M x(\lambda_{sM+j}),\end{equation}
\begin{equation}\label{uselater} \mathbb{E}[ f(x(\widehat{\lambda}_s))]=\frac{1}{M}\sum_{j=1}^M f(x(\lambda_{sM+j})),\end{equation}
where these identities are conditional on the past and on the realized inner-loop iterates; applying total expectation gives the unconditional identities used below.

It follows from \eqref{xs} that
\[
\mathbb{E}[x^S]
=\frac{\sum_{s=0}^{S-1}\frac{1}{\tau_{1,s}}\mathbb{E}[x(\widehat{\lambda}_s)]}{\sum_{s=0}^{S-1}\frac{1}{\tau_{1,s}}}
\overset{\eqref{uselater2}}{=}
\frac{\sum_{s=0}^{S-1}\frac{1}{\tau_{1,s}}\mathbb{E}\left[\frac{1}{M}\sum_{j=1}^{M}x(\lambda_{sM+j})\right]}{\sum_{s=0}^{S-1}\frac{1}{\tau_{1,s}}}
=\mathbb{E}[x_a^S].
\]
\begin{equation*}
\begin{aligned}
\mathbb{E}[f(x^S)]&\overset{\text{convexity of }f}{\leq}
\mathbb{E}\left[
\frac{\sum_{s=0}^{S-1}\frac{1}{\tau_{1,s}}f(x(\widehat{\lambda}_s))}{\sum_{s=0}^{S-1}\frac{1}{\tau_{1,s}}}
\right]\\
&\overset{\eqref{uselater}}{=}
\frac{\sum_{s=0}^{S-1}\frac{1}{\tau_{1,s}}\mathbb{E}\left[\frac{1}{M}\sum_{j=1}^{M}f(x(\lambda_{sM+j}))\right]}{\sum_{s=0}^{S-1}\frac{1}{\tau_{1,s}}}\\
&=\mathbb{E}[f_a^S].
\end{aligned}\end{equation*}

Using these identities, rearranging \eqref{defandconvex} gives
\begin{equation*}
\label{generalequivalent}
    \begin{aligned}
   &\frac{\tau_2M}{\tau_{1,S}^2}\left(\mathbb{E}[\phi(\widetilde{\lambda}^{S})]-\phi(\lambda^{\ast})\right)+\left(M\sum_{s=0}^{S-1}\frac{1}{\tau_{1,s}}\right) \left(\mathbb{E}[f(x^S)]-f(x(\lambda^{\ast}))\right)\\&\leq\left(\sum_{s=0}^{S-1}\frac{1}{\tau_{1,s}}M\right)\left\langle A\mathbb{E}[x^{S}]-b,\lambda\right\rangle +\frac{\tau_{2}M}{\tau_{1,0}^2}(\phi(\widetilde{\lambda}^0)-\phi(\lambda^{\ast}))\\
   &+\frac{1-\tau_{1,0}-\tau_2}{\tau_{1,0}^2} \left(\phi(y_{0})-\phi(\lambda^{\ast})\right)+9\overline{L}\left\Vert z_{0}-\lambda\right\Vert^2_2.
   \end{aligned}
\end{equation*}
\end{proof}

\section{Proofs for the Barycenter Results}\label{app:bc}
This supplement gives the full proofs of the results of Section~\ref{barycenter}. We use the notation
of Section~\ref{barycenter}. The variable $g_a$ denotes the scaled column dual variable, and the feasible
semi-dual subspace is $H=\{\mathbf g:\sum_a w_ag_a=0\}$. Write
$P^{(a)}_{ij}(g_a):=p^{(a)}_{ij}(g_a)$, where $p^{(a)}$ is defined in
\eqref{bc:softmax}, and set
\[
G_a(g_a):=\sum_i\mu_{a,i}h_i^{(a)}(g_a),
\qquad
\Phi(\mathbf g)=\sum_{a=1}^m w_aG_a(g_a).
\]
For a probability vector $p\in\Delta_n$ write
$\mathrm{Cov}(p)=\mathrm{diag}(p)-pp^\top$ for its categorical covariance and
$\mathrm{Var}_p(y)=y^\top\mathrm{Cov}(p)y=\sum_j p_jy_j^2-(\sum_jp_jy_j)^2$.

\subsection{Proof of Proposition~\ref{bc:gradid}}
\begin{proof}
\emph{Step 1. Uniqueness and the full gradient identity.}
The entropy term is strictly convex on each positive transport simplex, so the primal response in
\eqref{bc:plan} is unique. The Lagrangian dual objective is a pointwise supremum of affine
functions of the multipliers, and Danskin's theorem gives
\[
    \nabla_\lambda\Phi(\lambda)=Ax(\lambda)-b.
\]
Since the row constraints have already been eliminated,
\[
    \pi_a(\mathbf g)\mathbf1=\mu_a,\qquad a=1,\ldots,m,
\]
which proves the first displayed identity in Proposition~\ref{bc:gradid}.

\emph{Step 2. Block gradient of the semi-dual.}
Differentiating the row log-sum-exp term gives
\[
\nabla h_i^{(a)}(g_a)=P^{(a)}_{i\cdot}(g_a),
\]
where $P^{(a)}$ is the softmax matrix in \eqref{bc:softmax}. Therefore
\[
\nabla G_a(g_a)=\sum_i\mu_{a,i}P^{(a)}_{i\cdot}(g_a)
   =\pi_a(\mathbf g)^\top\mathbf1=:q_a(\mathbf g),
\]
and the ambient derivative of $\Phi(\mathbf g)=\sum_a w_aG_a(g_a)$ is
$\nabla_{g_a}\Phi=w_aq_a(\mathbf g)$.

\emph{Step 3. Restriction to the constraint subspace.}
The feasible semi-dual subspace is
\[
H:=\{\mathbf g=(g_1,\ldots,g_m):\sum_a w_ag_a=0\}.
\]
For $\delta\in H$ and any $\nu\in\mathbb R^n$,
\[
\sum_a w_a\delta_a=0,\qquad
\left.\frac{d}{dt}\Phi(\mathbf g+t\delta)\right|_{t=0}
=\sum_a\langle w_aq_a(\mathbf g),\delta_a\rangle
=\sum_a\langle w_a(q_a(\mathbf g)-\nu),\delta_a\rangle .
\]
Taking $\nu=\bar q(\mathbf g):=\sum_b w_bq_b(\mathbf g)$ gives
\[
\left.\frac{d}{dt}\Phi(\mathbf g+t\delta)\right|_{t=0}
=\sum_a\left\langle w_a\big(q_a(\mathbf g)-\bar q(\mathbf g)\big),\delta_a\right\rangle,
\qquad \delta\in H,
\]
which proves the second identity.

\emph{Step 4. Stationarity imposes a common column marginal.}
If all column marginals coincide, then $q_a(\mathbf g)=\bar q(\mathbf g)$ for all $a$, and the
derivative above is zero for every direction in $H$. Conversely, if the derivative is zero for every
direction in $H$, choose
$\delta_a=q_a(\mathbf g)-\bar q(\mathbf g)$. This direction belongs to $H$ because
$\sum_a w_a\delta_a=0$, and the displayed identity gives
\[
0=\left.\frac{d}{dt}\Phi(\mathbf g+t\delta)\right|_{t=0}
=\sum_a w_a\|q_a(\mathbf g)-\bar q(\mathbf g)\|_2^2.
\]
Thus $q_a(\mathbf g)=\bar q(\mathbf g)$ for every $a$. Hence stationarity on $H$ is exactly the
condition that the recovered plans share a common column marginal.

\emph{Step 5. Finite-sum form.}
The finite-sum representation in \eqref{bc:phi} is
\[
\Phi(\mathbf g)=\frac{1}{mn}\sum_{a=1}^m\sum_{i=1}^n mn\,w_a\mu_{a,i}h_i^{(a)}(g_a),
\]
with one component per pair $(a,i)$.
\end{proof}

\subsection{Proof of Proposition~\ref{bc:smooth}}\label{app:bc-smooth}
\begin{proof}
\emph{Step 1. Hessian of one row log-sum-exp.}
Fix $(a,i)$ and write $p=P^{(a)}_{i\cdot}(g_a)$. The gradient and Hessian of
$h_i^{(a)}$ are
\[
\nabla_{g_a}h_i^{(a)}=p,\qquad
\nabla^2_{g_a}h_i^{(a)}
=\frac{1}{\eta}\big(\mathrm{diag}(p)-pp^\top\big).
\]
For any $y\in\mathbb R^n$,
\[
y^\top\nabla^2_{g_a}h_i^{(a)}y
=\frac{1}{\eta}\left(\sum_j p_jy_j^2-\left(\sum_jp_jy_j\right)^2\right)
=\frac{1}{\eta}\operatorname{Var}_{p}(y).
\]
Therefore
\[
0\le y^\top\nabla^2_{g_a}h_i^{(a)}y
\le \frac{1}{\eta}\sum_jp_jy_j^2
\le \frac{1}{\eta}\|y\|_2^2,
\]
and hence
\[
0\preceq \nabla^2_{g_a}h_i^{(a)}\preceq \frac{1}{\eta} I_n.
\]

\emph{Step 2. Convexity and Euclidean smoothness of each component.}
The component
\[
\Phi_{a,i}(\mathbf g)=mn\,w_a\mu_{a,i}h_i^{(a)}(g_a)
\]
depends only on the block $g_a$. Its Hessian on $\mathbb R^{mn}$ is zero outside the $a$-th diagonal
block, and the nonzero block is
\[
mn\,w_a\mu_{a,i}\nabla^2_{g_a}h_i^{(a)}.
\]
Thus $\Phi_{a,i}$ is convex and
\[
\|\nabla^2\Phi_{a,i}\|_2
\le \frac{mn\,w_a\mu_{a,i}}{\eta}
=:L_{a,i}.
\]

\emph{Step 3. Average smoothness.}
Using $\mu_a\in\Delta_n$ and $\sum_a w_a=1$,
\[
\overline L=\frac{1}{mn}\sum_{a=1}^m\sum_{i=1}^n L_{a,i}
=\frac{1}{\eta}\sum_{a}w_a\sum_i\mu_{a,i}=\frac{1}{\eta}\sum_a w_a=\frac{1}{\eta} ,
\]
which proves the stated average parameter.

\emph{Step 4. $\ell_\infty$ smoothness of the full semi-dual.}
The Hessian of $\Phi=\sum_a w_aG_a$ is block diagonal.
\[
\nabla^2\Phi(\mathbf g)
=\operatorname{blkdiag}\big(w_a\nabla^2G_a(g_a)\big)_{a=1}^m.
\]
Let $u=(u_1,\ldots,u_m)$ with $\|u\|_\infty\le1$. For each row-softmax covariance,
\[
\operatorname{Var}_{P^{(a)}_{i\cdot}}(u_a)\le \|u_a\|_\infty^2\le1,
\]
so
\[
u_a^\top\nabla^2G_a(g_a)u_a
=\frac{1}{\eta}\sum_i\mu_{a,i}\mathrm{Var}_{P^{(a)}_{i\cdot}}(u_a)
\le\frac{1}{\eta} .
\]
Therefore
\[
u^\top\nabla^2\Phi\,u=\sum_a w_a\,u_a^\top\nabla^2 G_a(g_a)\,u_a
\le\frac{1}{\eta}\sum_a w_a=\frac{1}{\eta}\le\frac{5}{\eta} ,
\]
for every $u$ with $\|u\|_\infty\le1$.

Since $\nabla^2\Phi\succeq0$, Cauchy-Schwarz in the Hessian inner product gives, for any $u,v$ with
$\|u\|_\infty,\|v\|_\infty\le1$,
\[
|v^\top\nabla^2\Phi\,u|
\le
\sqrt{u^\top\nabla^2\Phi\,u}\sqrt{v^\top\nabla^2\Phi\,v}
\le \frac{1}{\eta} .
\]
Equivalently, $\|\nabla^2\Phi\,u\|_1\le1/\eta\|u\|_\infty$, so $\Phi$ is
$1/\eta$-smooth with respect to $\|\cdot\|_\infty$ on the ambient space and hence also on $H$.
The stated $L'=5/\eta$ is therefore a conservative valid bound.

\emph{Step 5. Block structure.}
For each block,
\[
\nabla^2G_a(g_a)
=\sum_i\mu_{a,i}\nabla^2_{g_a}h_i^{(a)}
\preceq \frac{1}{\eta}\left(\sum_i\mu_{a,i}\right)I
=\frac{1}{\eta} I.
\]
Consequently,
\[
\lambda_{\max}(\nabla^2\Phi)
=\max_a w_a\lambda_{\max}(\nabla^2G_a)
\le \frac{1}{\eta}.
\]
The smoothness quantities used in the rate are independent of $m$ and $n$, with
$\overline L=1/\eta$ exactly.
\end{proof}

\subsection{Auxiliary Estimates for Theorem~\ref{bc:main}}\label{app:bc-aux}
\paragraph{Variance estimate.}
\begin{lemma}[Variance bound]\label{bc:var}
With sampling $p_{a,i}=L_{a,i}/(N\overline L)$, the PDASGD-BC estimator $\widetilde\nabla$ satisfies
\[
\mathbb E\|\widetilde\nabla-\nabla\Phi(\lambda)\|_2^2
\le 2\overline L\,D_\Phi(\lambda,\widetilde\lambda)
\]
with $\overline L=1/\eta$.
\end{lemma}

\begin{proof}
\emph{Step 1. Estimator and unbiasedness.}
Write
\[
\Phi=\frac{1}{N}\sum_{(a,i)}\Phi_{a,i},\qquad N=mn,
\]
where each $\Phi_{a,i}$ is convex and $L_{a,i}$-smooth by Proposition~\ref{bc:smooth}. With
\[
p_{a,i}=\frac{L_{a,i}}{N\overline L},
\]
the estimator at $\lambda$ from snapshot $\widetilde\lambda$ is
\[
\widetilde\nabla
=\nabla\Phi(\widetilde\lambda)
+\frac{1}{Np_{a,i}}\big(\nabla\Phi_{a,i}(\lambda)-\nabla\Phi_{a,i}(\widetilde\lambda)\big).
\]
Taking expectation over $(a,i)$ gives
\[
\mathbb E[\widetilde\nabla]=\nabla\Phi(\lambda).
\]

\emph{Step 2. Second-moment bound.}
Using $\mathbb E\|X-\mathbb EX\|^2\le\mathbb E\|X\|^2$ with
\[
X=\frac{1}{Np_{a,i}}\big(\nabla\Phi_{a,i}(\lambda)-\nabla\Phi_{a,i}(\widetilde\lambda)\big),
\]
we obtain
\[
\mathbb E\|\widetilde\nabla-\nabla\Phi(\lambda)\|_2^2
\le\sum_{(a,i)}p_{a,i}\frac{1}{(Np_{a,i})^2}
\big\|\nabla\Phi_{a,i}(\lambda)-\nabla\Phi_{a,i}(\widetilde\lambda)\big\|_2^2.
\]
Since
\[
\frac{1}{Np_{a,i}}=\frac{\overline L}{L_{a,i}},
\]
this becomes
\[
\mathbb E\|\widetilde\nabla-\nabla\Phi(\lambda)\|_2^2
\le
\frac{\overline L}{N}\sum_{(a,i)}\frac{1}{L_{a,i}}
\big\|\nabla\Phi_{a,i}(\lambda)-\nabla\Phi_{a,i}(\widetilde\lambda)\big\|_2^2.
\]

\emph{Step 3. Co-coercivity and Bregman divergence.}
Let
\[
D_F(x,y):=F(x)-F(y)-\langle\nabla F(y),x-y\rangle .
\]
Convexity and $L_{a,i}$-smoothness imply
\[
\big\|\nabla\Phi_{a,i}(\lambda)-\nabla\Phi_{a,i}(\widetilde\lambda)\big\|_2^2
\le
2L_{a,i}D_{\Phi_{a,i}}(\lambda,\widetilde\lambda).
\]
Therefore
\[
\mathbb E\|\widetilde\nabla-\nabla\Phi(\lambda)\|_2^2
\le
\frac{2\overline L}{N}\sum_{(a,i)}D_{\Phi_{a,i}}(\lambda,\widetilde\lambda)
=2\overline L\,D_\Phi(\lambda,\widetilde\lambda),
\]
where the last equality uses linearity of the Bregman divergence in the averaged objective
$\Phi=1/N\sum_{(a,i)}\Phi_{a,i}$. With $\overline L=1/\eta$, this proves the variance bound.

\emph{Step 4. Restriction to the constraint subspace.}
Running on $H=\{\sum_a w_ag_a=0\}$ is equivalent to reduced coordinates
\[
\mathbf g=U\theta,\qquad U^\top U=I,\qquad \operatorname{range}(U)=H.
\]
Each component Hessian becomes $U^\top\nabla^2\Phi_{a,i}U$, and
\[
\lambda_{\max}\!\left(U^\top\nabla^2\Phi_{a,i}U\right)
\le
\lambda_{\max}\!\left(\nabla^2\Phi_{a,i}\right).
\]
Thus the same $L_{a,i}$ and sampling probabilities remain valid after restriction to $H$.

\emph{Step 5. Projected component cost.}
The Euclidean projection onto $H$ has the explicit form
\[
    (P_H v)_a=v_a-w_a r(v),\qquad
    r(v)=\frac{\sum_b w_bv_b}{\sum_b w_b^2}.
\]
If $v$ is a single component gradient supported on one block, then $r(v)$ is computed from one
$n$-vector. The terms $w_br(v)$ can be stored as a shared lazy offset, together with the local update
to the sampled block. Hence a projected inner component step costs one softmax and $\mathcal O(n)$
vector arithmetic. Full materialization is needed only for snapshot full gradients, whose cost is
already $\mathcal O(mn^2)$.
\end{proof}

\paragraph{Dual diameter estimate.}
\begin{lemma}[Dual diameter]\label{bc:diam}
Assume the mass lower bounds
$\min_{a,i}\mu_{a,i}\ge n^{-O(1)}$ and
$\min_j\nu^\star_j\ge n^{-O(1)}$, where $\nu^\star$ is the
entropic barycenter. Then the optimal dual point in
$H=\{\mathbf g:\sum_a w_ag_a=0\}$ satisfies
$D^2=\|\mathbf g^\star\|_2^2\le 4mn(R')^2
=\widetilde{\mathcal O}_\eta\!\left(mn(1+\max_a\|C_a\|_\infty)^2\right)$, where
$R'=\widetilde{\mathcal O}_\eta(1+\max_a\|C_a\|_\infty)$ is the per-subproblem $\ell_\infty$ dual bound obtained by
applying Lemma~3.2 of \citealt{lin2019efficient} to each separable block.
\end{lemma}

\begin{proof}
\emph{Step 1. Per-subproblem dual bound.}
For each $a$, consider the two-marginal entropic OT problem between $\mu_a$ and the optimal barycenter
$\nu^\star$. Lemma~3.2 of \cite{lin2019efficient} gives an optimal centered column dual variable
$\widetilde g_a$ satisfying
\[
\|\widetilde g_a\|_\infty\le
R'_a:=\eta\left(\frac{\|C_a\|_\infty}{\eta}+\ln n-2\ln\min_{i,j}\{\mu_{a,i},\nu^\star_j\}+\frac{1}{2}\right).
\]
Under the stated polynomial lower bounds,
\[
R':=\max_aR'_a
=\widetilde{\mathcal O}_\eta(1+\max_a\|C_a\|_\infty).
\]

\emph{Step 2. Align the dual gauges.}
Use the equivalent standard barycenter formulation with the redundant constraint $\nu\in\Delta_n$.
Since $\nu^\star$ has strictly positive entries, the KKT condition for minimizing the barycenter
objective over the simplex gives
\[
\sum_{a=1}^m w_a\widetilde g_a=\rho\,\mathbf 1
\]
for some scalar $\rho$. Since $\|\widetilde g_a\|_\infty\le R'$ for all $a$,
\[
|\rho|
=\left\|\sum_a w_a\widetilde g_a\right\|_\infty
\le \sum_a w_a\|\widetilde g_a\|_\infty
\le R'.
\]
The column dual gauge shift does not change any transport plan. Define
\[
g_a^\star:=\widetilde g_a-\rho\,\mathbf1.
\]
Then
\[
\sum_a w_ag_a^\star
=\sum_aw_a\widetilde g_a-\rho\mathbf1
=0,
\]
so $\mathbf g^\star=(g_1^\star,\ldots,g_m^\star)\in H$ is an optimal semi-dual point.
Moreover
\[
\|g_a^\star\|_\infty
\le
\|\widetilde g_a\|_\infty+|\rho|
\le 2R'.
\]

\emph{Step 3. Euclidean diameter.}
Thus
\[
D^2=\|\mathbf g^\star\|_2^2=\sum_a\|g_a^\star\|_2^2\le 4mn(R')^2
=\widetilde{\mathcal O}_\eta\!\left(mn(1+\max_a\|C_a\|_\infty)^2\right).
\]
\end{proof}

\subsection{Proof of Theorem~\ref{bc:main}}\label{app:proof-bc-main}
\begin{proof}
\emph{Step 1. Apply the general theorem on the constraint subspace.}
Let
\[
H=\{\mathbf g:\sum_a w_ag_a=0\}.
\]
Choose an orthonormal matrix $U$ with $\operatorname{range}(U)=H$ and write
\[
\mathbf g=U\theta,\qquad \Psi(\theta)=\Phi(U\theta).
\]
The Euclidean prox geometry is unchanged under this parametrization, so Theorem~\ref{generaltheorem}
and Corollary~\ref{theorem1cor} apply to $\Psi$.

\emph{Step 2. Plug in the barycenter constants.}
By Proposition~\ref{bc:gradid}, Proposition~\ref{bc:smooth}, and Lemma~\ref{bc:var},
\[
N=mn,\qquad
\overline L=\frac{1}{\eta},\qquad
L'=\frac{5}{\eta}.
\]
Lemma~\ref{bc:diam} gives an optimal semi-dual point $\mathbf g^\star\in H$ such that
\[
\|\mathbf g^\star\|_\infty\le 2R',
\qquad
\|\mathbf g^\star\|_2^2\le 4mn(R')^2,
\qquad
R'=\widetilde{\mathcal O}_\eta(1+\max_a\|C_a\|_\infty).
\]

\emph{Step 3. Objective residual after $S$ outer loops.}
Set $M=N=mn$. Corollary~\ref{theorem1cor} yields
\[
\mathbb E[f(x^S)]-f^\star
=\mathcal O\!\left(
\frac{L'(R')^2}{S^2}
+\frac{\overline L\,mn(R')^2}{mn\,S^2}
\right)
=\widetilde{\mathcal O}_\eta\!\left(\frac{(1+\max_a\|C_a\|_\infty)^2}{S^2}\right).
\]
Thus it is sufficient to take
\[
S=\widetilde{\mathcal O}_\eta\!\left(\frac{1+\max_a\|C_a\|_\infty}{\sqrt{\epsilon}}\right).
\]

\emph{Step 4. Arithmetic cost per outer loop.}
One outer loop consists of
\[
\text{one snapshot full gradient}
\quad+\quad
M=mn\text{ inner component gradients}.
\]
Each component gradient costs $\mathcal O(n)$, so
\[
\text{outer-loop cost}
=\mathcal O(mn\cdot n)+\mathcal O(mn\cdot n)
=\mathcal O(mn^2).
\]

\emph{Step 5. Total complexity.}
Multiplying the number of outer loops by the outer-loop cost gives
\[
\widetilde{\mathcal O}_\eta\!\left(\frac{1+\max_a\|C_a\|_\infty}{\sqrt{\epsilon}}\right)\cdot\mathcal O(mn^2)
=\widetilde{\mathcal O}_\eta\!\left(\frac{mn^2(1+\max_a\|C_a\|_\infty)}{\sqrt{\epsilon}}\right).
\]
\end{proof}

\begin{remark}[Alternative component grouping]\label{bc:support-index-grouping}
Grouping all terms with the same support index $i$ into a single component would give
$N=n$ components, but each component would differentiate all $m$ blocks and cost $\mathcal O(mn)$. The
same rate calculation yields
$\widetilde{\mathcal O}_\eta(m^{1.5}n^2/\sqrt{\epsilon})$, so the separate $(a,i)$ components in
\eqref{bc:phi} are preferable.
\end{remark}

\subsection{Proof of Proposition~\ref{bc:feas-main}}\label{app:proof-bc-feas}
\begin{proof}
For the proof set $W:=\sum_{s=0}^{S-1}1/\tau_{1,s}$ and $L':=5/\eta$. For a collection of plans
$x=(\pi_1,\ldots,\pi_m)$, define
\[
\bar\nu(x):=\sum_b w_b\pi_b^\top\mathbf1,\qquad
\mathcal R(x)_a:=w_a\big(\pi_a^\top\mathbf1-\bar\nu(x)\big).
\]
The bound in \eqref{bc:feas-bound} is equivalent to a bound on
$\mathbb E\|\mathcal R(x^S)\|_1$. We prove the intermediate estimate
\[
\mathbb E\|\mathcal R(x^S)\|_1
\le
2\sqrt{2L'}\,\frac{1}{W}\sum_{s=0}^{S-1}\frac{\sqrt{\delta_s}}{\tau_{1,s}},
\]
which implies \eqref{bc:feas-bound} by Cauchy-Schwarz.

\emph{Step 1. Linearity of the residual under primal averaging.}
Write $q_a(\mathbf g):=\pi_a(\mathbf g)^\top\mathbf1$ and
$\bar\nu(\mathbf g):=\sum_b w_bq_b(\mathbf g)$. For a primal response $x(\mathbf g)$ define
\[
\mathcal R(x(\mathbf g))_a=w_a(q_a(\mathbf g)-\bar\nu(\mathbf g)).
\]
The row constraints are satisfied exactly by \eqref{bc:plan}, so the residual has no row-marginal
component.
Because column marginals and $\bar\nu$ are linear in the transport plans, the residual is linear in the
primal response. For the weighted output
\[
x^S=\frac{1}{W}\sum_{s=0}^{S-1}\frac{1}{\tau_{1,s}}x(\widehat{\mathbf g}_s),
\qquad W=\sum_{s=0}^{S-1}\frac{1}{\tau_{1,s}},
\]
we have
\[
\mathcal R(x^S)=\frac{1}{W}\sum_{s=0}^{S-1}\frac{1}{\tau_{1,s}}\mathcal R(x(\widehat{\mathbf g}_s)).
\]
The triangle inequality gives
\begin{equation}\label{app:feas-triangle}
\mathbb E\|\mathcal R(x^S)\|_1
\le
\frac{1}{W}\sum_{s=0}^{S-1}\frac{1}{\tau_{1,s}}\,
\mathbb E\|\mathcal R(x(\widehat{\mathbf g}_s))\|_1 .
\end{equation}

\emph{Step 2. Reduce fixed-point residual to distance from the optimal barycenter marginal.}
Fix a semi-dual point $\mathbf g\in H$. Let
$\mathbf g^\ast$ be an optimal point and let
$\nu^\star:=q_a(\mathbf g^\ast)$, which is independent of $a$ by Proposition~\ref{bc:gradid}. By
definition of $\bar\nu$,
\[
\sum_a w_a\|q_a(\mathbf g)-\bar\nu(\mathbf g)\|_1
\le
\sum_a w_a\|q_a(\mathbf g)-\nu^\star\|_1+\|\bar\nu(\mathbf g)-\nu^\star\|_1
\le
2\sum_a w_a\|q_a(\mathbf g)-\nu^\star\|_1 .
\]
Since $\|\mathcal R(x(\mathbf g))\|_1=\sum_aw_a\|q_a(\mathbf g)-\bar\nu(\mathbf g)\|_1$, it remains to
control the deviations $q_a(\mathbf g)-\nu^\star$.

\emph{Step 3. Bregman control of the column marginals.}
For each block $G_a$, $\nabla G_a(g_a)=q_a(g_a)$ and $G_a$ is $L'$-smooth with respect to
$\|\cdot\|_\infty$ for the valid bound $L'=5/\eta$. The standard smooth-convex Bregman inequality for
the norm pair $(\|\cdot\|_\infty,\|\cdot\|_1)$ gives
\[
\|q_a(g_a)-q_a(g_a^\ast)\|_1^2
\le
2L'\Big(G_a(g_a)-G_a(g_a^\ast)-\langle q_a(g_a^\ast),g_a-g_a^\ast\rangle\Big).
\]
Since $q_a(g_a^\ast)=\nu^\star$ for all $a$ and both $\mathbf g,\mathbf g^\ast$ lie in $H$,
\[
\sum_a w_a\langle \nu^\star,g_a-g_a^\ast\rangle
=\left\langle \nu^\star,\sum_a w_a(g_a-g_a^\ast)\right\rangle=0.
\]
Therefore the weighted sum of the block Bregman divergences equals the global dual gap.
\[
\sum_a w_a\Big(G_a(g_a)-G_a(g_a^\ast)-\langle \nu^\star,g_a-g_a^\ast\rangle\Big)
=\Phi(\mathbf g)-\Phi(\mathbf g^\ast).
\]
By Cauchy-Schwarz with weights $w_a$,
\[
\sum_a w_a\|q_a(\mathbf g)-\nu^\star\|_1
\le
\sqrt{\sum_a w_a\|q_a(\mathbf g)-\nu^\star\|_1^2}
\le
\sqrt{2L'\big(\Phi(\mathbf g)-\Phi(\mathbf g^\ast)\big)}.
\]
Combining the preceding displays,
\begin{equation}\label{app:feas-fixed}
\|\mathcal R(x(\mathbf g))\|_1
\le
2\sqrt{2L'\big(\Phi(\mathbf g)-\Phi(\mathbf g^\ast)\big)} .
\end{equation}

\emph{Step 4. Pass from fixed points to the randomized iterates.}
Apply \eqref{app:feas-fixed} to $\mathbf g=\widehat{\mathbf g}_s$ and use Jensen's inequality.
\[
\mathbb E\|\mathcal R(x(\widehat{\mathbf g}_s))\|_1
\le
2\sqrt{2L'}\,\mathbb E\sqrt{\Phi(\widehat{\mathbf g}_s)-\Phi(\mathbf g^\ast)}
\le
2\sqrt{2L'\delta_s}.
\]
Substituting this into \eqref{app:feas-triangle} yields
\[
\mathbb E\|\mathcal R(x^S)\|_1
\le
2\sqrt{2L'}\,\frac{1}{W}\sum_{s=0}^{S-1}\frac{\sqrt{\delta_s}}{\tau_{1,s}}.
\]

\emph{Step 5. Weighted Cauchy-Schwarz.}
Cauchy-Schwarz gives
\[
\sum_{s=0}^{S-1}\frac{\sqrt{\delta_s}}{\tau_{1,s}}
\le
\sqrt{\sum_{s=0}^{S-1}\frac{1}{\tau_{1,s}}}
\sqrt{\sum_{s=0}^{S-1}\frac{\delta_s}{\tau_{1,s}}}
=\sqrt W\sqrt{\sum_{s=0}^{S-1}\frac{\delta_s}{\tau_{1,s}}},
\]
which proves Proposition~\ref{bc:feas-main}. The proof uses only
smoothness, convexity, the common-marginal KKT condition, and linearity of the primal averaging; it does not
use strong convexity of $\Phi$ or any variance bound on the primal iterates.
\end{proof}

\bibliographystyle{elsarticle-harv}
\bibliography{OT}

\end{document}